\documentclass{article}
\pdfoutput = 1

\usepackage[accepted]{icml2021}

\usepackage[utf8]{inputenc} 
\usepackage[T1]{fontenc}    
\usepackage{hyperref}       
\usepackage{amsfonts}       
\usepackage{nicefrac}       
\usepackage{microtype}      
\usepackage{lipsum}

\usepackage{times}
\usepackage{graphicx}
\usepackage{amsmath}
\usepackage[perpage]{footmisc}

\usepackage{mathtools}

\newcommand{\norm}[1]{\left\lVert#1\right\rVert}

\newcommand{\Dcal}[0]{\mathcal{D}}
\newcommand{\Lcal}[0]{\mathcal{L}}

\newcommand{\Ocal}[0]{\mathcal{O}}

\newcommand{\expect}{\operatorname{\mathbb{E}}}
\newcommand{\Var}{\operatorname{Var}}

\newcommand{\VAR}[1]{\Var\left[#1\right]}
\newcommand{\VARR}[2]{\Var_{#1}\left[#2\right]}

\newcommand{\E}[1]{\expect\left[#1\right]}
\newcommand{\EE}[2]{\expect_{#1}\left[#2\right]}

\usepackage[english]{babel}
\usepackage{amsthm}
\newtheorem{assumption}{Assumption}
\newtheorem{theorem}{Theorem}

\newtheorem{lemma}[theorem]{Lemma}
\newtheorem{proposition}{Proposition}

\usepackage{tikz}  
\usetikzlibrary{arrows}
\usepackage{pgfplots}
\pgfplotsset{compat=1.15}
\usepackage{mathrsfs}

\usepackage[colorinlistoftodos,prependcaption,textsize=tiny]{todonotes}

\DeclarePairedDelimiter\floor{\lfloor}{\rfloor}

\icmltitlerunning{Clustered Sampling:  Low-Variance and Improved Representativity for Clients Selection in Federated Learning}

\begin{document}

\twocolumn[

\icmltitle{Clustered Sampling:  Low-Variance and Improved Representativity \\for Clients Selection in Federated Learning}




\begin{icmlauthorlist}
\icmlauthor{Yann Fraboni}{to,goo}
\icmlauthor{Richard Vidal}{goo}
\icmlauthor{Laetitia Kameni}{goo}
\icmlauthor{Marco Lorenzi}{to}
\end{icmlauthorlist}

\icmlaffiliation{to}{Universit\'e C\^{o}te d’Azur, Inria, Epione Research Group, France}
\icmlaffiliation{goo}{Accenture Labs, Sophia Antipolis, France}

\icmlcorrespondingauthor{Yann Fraboni}{yann.fraboni@inria.fr}

\icmlkeywords{Machine Learning, ICML, Federated Learning, Clustered Sampling, client sampling, MD sampling, client sampling, Multonomial, Distribution, unbiased, variance, reduction, representativity, clients selection, FL}

\vskip 0.3in
]



\printAffiliationsAndNotice{}  

\begin{abstract}

This work addresses the problem of optimizing communications between server and clients in federated learning (FL). Current sampling approaches in FL are either biased, or non optimal in terms of server-clients communications and training stability. To overcome this issue, we introduce \textit{clustered sampling} for clients selection. We prove that clustered sampling leads to better clients representatitivity and to reduced variance of the clients stochastic aggregation weights in FL. Compatibly with our theory, we provide two different clustering approaches enabling clients aggregation based on 1) sample size, and 2) models similarity. Through a series of experiments in non-iid and unbalanced scenarios, we demonstrate that model aggregation through clustered sampling consistently leads to better training convergence and variability when compared to standard sampling approaches. Our approach does not require any additional operation on the clients side, and can be seamlessly integrated in standard FL implementations. Finally, clustered sampling is compatible with existing methods and technologies for privacy enhancement, and for communication reduction through model compression.

\end{abstract}

\section{Introduction}

Federated learning (FL) is a training paradigm enabling different clients to jointly learn a global model without sharing their respective data. 
Communication can be a primary bottleneck for FL since wireless and other end-user internet connections operate at variable communication rates while being potentially unreliable. 
Moreover, the capacity of the aggregating server may impose constraints on the number of clients the server can communicate with at the same time.
These considerations led to significant interest in reducing the number and bandwidth of communications at every step of the FL process.

One of the most popular communication reduction strategies consists in limiting the frequency of communications at the expense of increased computation on the clients side. This is usually achieved by asking the clients to perform multiple iterations of local gradient descent before communicating their updates. In this setting, FedAvg \cite{FedAvg} is the first and most widely used FL algorithm, for which convergence bounds were given in \cite{FedNova, ontheconvergence, SCAFFOLD, Yu_Yang_Zhu_2019, pmlr-v108-bayoumi20a, pmlr-v119-woodworth20a,  Lin2020Don't, stich2018local}.


To further reduce the number of communications, the server can select a subset of clients participating at every iteration. This strategy, called client sampling, enables reducing communications to the minimum. FedAvg first proposed selecting $m$ clients uniformly without replacement while replacing the contribution of the non-sampled clients with the current global model. However this scheme is known for being biased, since the resulting model is, in expectation, different from the deterministic aggregation of every client. To overcome this issue, \cite{FedProx} proposes an \textit{unbiased} sampling scheme where the new global model is created as the average of the sampled clients work. The sampling is based on a multinomial distribution (MD) whose clients probabilities corresponds to their relative sample size. While other clients sampling schemes have been proposed, most of them require additional server-clients communications and are not proven to be unbiased \cite{ontheconvergence, Richtarik_optimal_sampling, sampling_mobile_edge}.

To the best of our knowledge, FedAvg and MD sampling are the only schemes keeping to a minimum server-clients communications.
In particular, MD sampling has been proven to lead to FL optimum and shown experimentally to outperform FedAvg sampling \cite{ontheconvergence}. In spite of its unbiasedness, MD sampling may still lead to large variance in the clients selection procedure. 
While unbiasedness guarantees proper clients representativity in expectation, representativity is not necessarily achieved when considering a single FL iteration. Since at each MD sampling instance we select clients with replacement, this determines a variance in the amount of times a client is selected. This sampling variance is a leading cause of the large variability in the convergence of FL, especially in non-iid applications.
Indeed, at each iteration, sampled clients improve the global model based on their data distribution, to the detriment of the data specificity of non-sampled clients. 

While the literature mainly focused on the study of the behavior of FL sampling strategies in expectation, to our knowledge this study provides the first theoretical investigation of the variability properties of FL sampling. In what follows, we show that this statistical aspect is crucial to determine convergence stability and quality of FL. 
The contribution of this work is the introduction of \textit{clustered sampling}, a new unbiased client sampling scheme improving MD sampling by guaranteeing smaller client selection variability, while keeping to a minimum server-clients communications. By increasing every client representativity in model aggregations, clustered sampling ensures that clients with unique distributions are more likely of being sampled, leading to smoother and faster FL convergence.

We first derive, in Section \ref{sec:related}, the theory behind current FL sampling schemes our work is built on. We then formally introduce clustered sampling in Section \ref{sec:clustered_sampling} and prove its theoretical correctness by extending the work done in \cite{FedNova}. We finally show the theoretical benefits of clustered sampling over MD sampling. 
In Section \ref{sec:basic_clustered}, we propose an implementation of clustered sampling aggregating clients based on their sample size, showing that this approach leads to reduced variance of the clients' aggregation weights.
In Section \ref{sec:clustered_improved}, we extend our sampling theory to aggregation schemes based on the similarity between clients updates, showing that this approach further reduces the variance of clients aggregation weights while improving the representation of the clients during each FL aggregation step, as compared to MD sampling. This result leads to an overall improvement of the convergence of FL.
Finally, in Section \ref{sec:experiments}, we experimentally demonstrate our work on a broad range of balanced and unbalanced heterogeneous dataset. The code used for this work is available \href{https://github.com/Accenture/Labs-Federated-Learning/tree/clustered_sampling}{here}\footnote{https://github.com/Accenture//Labs-Federated-Learning/tree/clustered\_sampling}.

\section{Related Work}
\label{sec:related}

Before introducing in Section \ref{sec:clustered_sampling} the core idea of clustered sampling, we first recapitulate in Section \ref{sec:related} the current theory behind parameter aggregation and sampling schemes for FL.

\subsection{Aggregating clients local updates}


In FL, we consider a set $I$ of clients respectively owning  datasets $\Dcal_i$ composed of $n_i$ samples. FL aims at optimizing the average of each clients local loss function weighted by their importance $p_i$
\begin{equation}
\label{eq:global_loss}
	\Lcal(\theta) = \sum_{i\in I}p_i\Lcal_i(\theta),
\end{equation}
where $\theta$ represents the model parameters and $\sum_{i=1}^n p_i  = 1$. While any combination of $\{p_i \}$ is possible, a common choice consists in defining $p_i = n_i /M$, where $M = \sum_{i \in I} n_i$ is the total number of sample across datasets. In this work, we adopt the same definition of the improtance weights, althought the theory derived below does not depend on ny specific choice of the parameters $\{p_i\}$.

In this setting, to estimate a global model across clients, FedAvg \cite{FedAvg} is an iterative training strategy based on the aggregation of local model  parameters ${\theta}_i^t$. At each iteration step $t$, the server sends the current global model parameters $\theta^t$ to the clients. Each client updates the model by minimizing the local cost function $\Lcal({\theta}_i^{t+1},\Dcal_i)$ through a fixed amount of SGD initialized with $\theta^t$. Subsequently each client returns the updated local parameters ${\theta}_i^{t+1}$ to the server. The global model parameters $\theta^{t+1}$ at the iteration step $t+1$ are then estimated as a weighted average, i.e.
\begin{equation}
\label{FedAvg_server_aggregation}
\theta^{t+1}=\sum_{i\in I}\frac{n_i}{M}{\theta}_i^{t+1}.
\end{equation}

\subsection{Clients' sampling}

Clients sampling is a central operation of FL. FedAvg \cite{FedAvg} proposes to uniformly sample a subset of participating clients $S_t$ at every iteration while the other clients updates are replaced by the current global model, i.e.
\begin{equation}
\label{FedAvg_server_aggregation_sampling}
\theta^{t+1}
= \sum_{i\in S_t}\frac{n_i}{M}{\theta}_i^{t+1} + \sum_{i\notin S_t}\frac{n_i}{M}{\theta}^{t}.
\end{equation}

The sampling scheme introduced by FedAvg is generally slow due to the attrition introduced by non-participating clients. To solve this problem, \cite{FedProx} proposes instead to sample $S_t$, the subset of clients at iteration $t$, from a Multinomial Distribution (MD) where each client is sampled according to its relative data ratio $p_i = \frac{n_i}{M}$. The new global model is obtained as the average of each selected client, i.e.
\begin{equation}
\label{FedProx_server_aggregation_sampling}
\theta^{t+1}
= \sum_{i\in S_t}\frac{1}{m}{\theta}_i^{t+1} .
\end{equation}
By design, MD sampling is such that the aggregation of clients model updates is identical in expectation to the one obtained when considering all the clients, i.e. $\EE{S_t}{\theta^{t+1}}=\sum_{i\in I}p_i\theta_i^{t+1}$. Sampling schemes following this property are called \textit{unbiased}. Notably, the sampling scheme employed by FedAvg does not satisfy this property, and it is thus prone to clients-drift \cite{SCAFFOLD}.

\subsection{FL convergence with MD client sampling}

Theoretical guarantees regarding the convergence of FedAvg were given in \cite{FedNova}. The proof relies on assumptions classically used in Stochastic Gradient Descent (SGD) analysis \cite{Bottou2018} (Assumptions \ref{ass:smoothness} and \ref{ass:unbiased} below), or commonly used in the federated optimization literature \cite{FedProx, ontheconvergence, Haddadpour2019, SCAFFOLD, MATCHA} to capture the dissimilarities of local objectives (Assumption \ref{ass:dissimilarity} below). 



\begin{assumption}[Smoothness]\label{ass:smoothness}
	The clients local objective function is Lipschitz smooth, that is, $\norm{\nabla \Lcal_i(x) - \nabla \Lcal_i(y)} \le L \norm{x - y}, \forall i \in \{1, ..., n\}$.
\end{assumption}

\begin{assumption}[Unbiased Gradient and Bounded Variance]\label{ass:unbiased}
	For each client $i$ local model, the stochastic gradient $g_i(x|\xi)$ of model $x$ evaluated on batch $\xi$ is an unbiased estimator of the local gradient: $\EE{\xi}{g_i(x|\xi)} = \nabla \Lcal_i(x)$, and has bounded variance $\EE{\xi}{\norm{g_i(x|\xi) - \nabla \Lcal_i(x)}} \le \sigma^2, \forall i \in \{1,..., n\}$ with $\sigma^2\ge 0$.
\end{assumption}

\begin{assumption}[Bounded Dissimilarity]\label{ass:dissimilarity}
	For any set of weights $\{ w_i\ge 0\}_{i=1}^n$ such that $\sum_{i=1}^{n}w_i =1$, there exists constants $\beta^2\ge 1$ and $\kappa^2 \ge 0$ such that $\sum_{i=1}^{n}w_i \norm{\nabla \Lcal_i(x)}^2 \le \beta^2 \norm{\sum_{i=1}^{n}w_i \nabla \Lcal_i(x)}^2 + \kappa^2$. If all the local loss functions are identical, then we have $\beta^2 =1$ and $\kappa^2 =0$. 
\end{assumption}

The following theorem was proven in \cite{FedNova} and provides theoretical guarantees for MD client sampling.

\begin{theorem}\label{theo:convergence_FP}
	Under Assumptions \ref{ass:smoothness} to \ref{ass:dissimilarity}, and local learning rate $\eta = \sqrt{m/NT}$, 
	FL with $FedAvg$ when sampling $m$ clients with MD
	converges to a stationary point of $\Lcal(\theta)$:
	\begin{equation}
		\frac{1}{T}\sum_{t=0}^{T}\E{\norm{\nabla \Lcal(\theta^t)}^2}
		\le \Ocal(\frac{1}{\sqrt{mNT}}) + \Ocal(\frac{mN}{T}).
	\end{equation}
\end{theorem}
The proof of Theorem \ref{theo:convergence_FP} can be found in \cite{FedNova} and shows that considering a subset of workers with MD client sampling is enough to ensure convergence of the global model to a local minimum of the federated loss function, equation (\ref{eq:global_loss}). Following the conclusions of that work, to avoid optimizing a surrogate loss function instead of the federated one in equation (\ref{eq:global_loss}), the server asks from every client to compute the same amount of SGD steps $N$. 


\subsection{Sampling schemes comparison}

Other client sampling schemes have been proposed. For example, with \cite{ontheconvergence}, the server sends the global model to every client before creating the new global model out of the first $m$ updated models the server receives; with \cite{Richtarik_optimal_sampling}, the server waits for every client to send the norm of their work before selecting the $m$ clients with the most relevant updates; with \cite{sampling_mobile_edge}, clients transmit information about their available computation resources before the server selects $m$ of them in function of their availability. Contrarily to FedAvg and MD sampling, these sampling schemes require additional communications and sometimes even computation from all the clients. On the contrary, FedAvg and MD sampling have the appealing property of maintaining to a minimum the amount of clients-server communications at each iteration. The global model is sent only to the sampled clients, and the amount of local work is set for those clients to only $N$ SGD updates. 
To the best of our knowledge, FedAvg and MD sampling schemes are the only approaches minimizing the effective work and communications asked to the clients.
Moreover, while FedAvg sampling has no theoretical guarantees regarding its convergence, neither regarding the unbiasedness of its global model, MD sampling is shown to converge to a stationary point of the global loss function (\ref{eq:global_loss}) (Theorem \ref{theo:convergence_FP}).  Based on these considerations and given that \cite{ontheconvergence} shows experimentally that MD sampling outperforms FedAvg, in the rest of this work we consider MD sampling as reference sampling technique.

\section{Clustered sampling}\label{sec:clustered_sampling}

In Section \ref{sec:clustered_sampling_introduction}, we first introduce clustered sampling and prove the convergence of FL under this scheme. In Section \ref{sec:clustered_sampling_improvments}, we show the statistical improvements brought by clustered sampling as compared to MD client sampling, in terms of reduced sampling variance, and better clients representativity across the entire FL process.

\subsection{Definition of clustered sampling}\label{sec:clustered_sampling_introduction}


Let us consider $n$ clients participating to FL. With MD sampling, $m$ clients are sampled from a multinomial distribution supported on $\{1,..., n\}$ where a client is selected in function of its  data ratio $p_i$. 
\begin{assumption}[Unbiased Sampling]\label{ass:sampling} 
    A client sampling scheme is said unbiased if the expected value of the client aggregation is equal to the global deterministic aggregation obtained when considering all the clients, i.e.
	\begin{equation}
	\EE{S_t}{\theta^t}
	=
	\EE{S_t}{\sum_{j \in S_t}w_j(S_t)\theta_j^t}
	\coloneqq
	\sum_{i =1 }^n p_i\theta_i^t ,
	\end{equation}
	where $w_j(S_t)$ is the aggregation weight of client $j$ for subset of clients $S_t$.
\end{assumption}

In \cite{ontheconvergence}, the notion of unbiased sampling is introduced by means of Assumption \ref{ass:sampling}. MD sampling follows this assumption, and thus provides at every iteration an unbiased global model. However, MD sampling enables a client to be sampled from $0$ to $m$ times with non-null probability at each iteration, giving aggregation weights for every client ranging from 0 to 1. 
As a result, MD sampling provides appropriate representation for every client in expectation, with however potentially large variance in the amount of times a client is selected. As a consequence, the representativity of a client at any given realization of a FL iteration may not guaranteed. In the following we introduce clustered sampling, and show that this strategy leads to decreasing clients aggregation weight variance and better clients representativity.
%
%

We denote by $W_0$ the multinomial distribution with support on $\{1, ..., n\}$ used to sample one client according to its data ratio $p_i$. MD sampling can be seen as sampling $m$ times with $W_0$. With clustered sampling, we propose to generalize MD sampling by sampling $m$ clients according to $m$ independent distributions $\{W_k(t)\}_{k=1}^m$ each of them privileging a different subset of clients based on opportune selection criteria (Section \ref{sec:basic_clustered} and \ref{sec:clustered_improved}). With clustered sampling, the $m$ clients can be sampled with different distributions and, at two different iterations, the set of distributions can differ. MD sampling is a special case of clustered sampling when $\forall t,\ \forall k\in\{1, ..., m\},\ W_k(t) = W_0$.

In the rest of this work, we denote by $r_{k, i}^t$ the probability for client $i$ to be sampled in distribution $W_k(t)$. By construction, we have:
\begin{equation}
\forall k \in\{1, ..., m\},\ 
\sum_{i=1}^{n}r_{k, i}^t = 1 \text{ with } r_{k, i}^t \ge 0
\label{eq:sum_proba}
.
\end{equation} 

We also require clustered sampling to be unbiased. Extending Assumption \ref{ass:sampling} to $m$ independent sampling distributions $\{W_k(t)\}_{k=1}^m$, we obtain the property:
\begin{equation}
\label{eq:total_probability}
\forall i \in \{1, ..., n\},\ \sum_{k=1}^{m} r_{k,i}^t = m\ p_i .
\end{equation}

\begin{proposition}\label{prop:equivalent}
Equations (\ref{eq:sum_proba}) and (\ref{eq:total_probability}) are sufficient conditions for clustered sampling to satisfy Assumption \ref{ass:sampling}.
\end{proposition}

\begin{proof}
Satisfying equation (\ref{eq:sum_proba}) ensures the $m$ distributions used with clustered sampling are feasible. When sampling one client from one of the $m$ distributions $W_k(t)$, we get: 	
\begin{equation}
	\EE{W_k(t)}{\sum_{j \in W_k(t)}w_j(W_k(t))\theta_j^t}
	=
	\sum_{i =1 }^n r_{k, i}^t\theta_i^t .
	\label{eq:averages_distribution}
\end{equation}

By linearity of the expected value, the expected new global model
is the average between the weighted models obtained according to each distribution $\{W_k\}_{k=1}^m$ derived in equation (\ref{eq:averages_distribution}), i.e.
\begin{align}
	\EE{S_t}{\theta^t}
	&=
	\sum_{k = 1}^m\frac{1}{m}{\sum_{i =1 }^n r_{k, i}^t\theta_i^t }=	\sum_{i =1 }^n p_i\theta_i^t
	,
\end{align}	
where the second equality comes from equation (\ref{eq:total_probability}).
\end{proof}

In Theorem \ref{theo:convergence_clustered}, we prove that FedAvg with clustered sampling satisfying Assumptions \ref{ass:smoothness} to \ref{ass:dissimilarity} and Proposition \ref{prop:equivalent} has the same convergence bound to a FL local optimum as with FedAvg and MD sampling. The proof of Theorem \ref{theo:convergence_clustered} can be found in Appendix \ref{app:proof_theo_2}.

\begin{theorem}\label{theo:convergence_clustered}
	Under Assumption \ref{ass:smoothness} to \ref{ass:dissimilarity}, and local learning rate $\eta = \sqrt{m/NT}$, 
	let's consider FL with $FedAvg$ when sampling $m$ clients with clustered sampling scheme satisfying Proposition \ref{prop:equivalent}. 	The same asymptotic behavior of MD sampling holds :
	\begin{equation}
	\frac{1}{T}\sum_{t=0}^{T}\E{\norm{\nabla \Lcal(\theta^t)}^2}
	\le \Ocal(\frac{1}{\sqrt{mNT}}) + \Ocal(\frac{mN}{T})
	\end{equation}
\end{theorem}

In the proof of Theorem \ref{theo:convergence_clustered} (Appendix \ref{app:proof_theo_2}), we show that this convergence bound holds for any clustered sampling scheme satisfying Proposition \ref{prop:equivalent}. Moreover, the convergence bound of MD sampling is a bound itself for the convergence of a general clustered sampling scheme satisfying Proposition \ref{prop:equivalent}. Therefore, clustered sampling enjoys better convergence guarantees than MD sampling.

\subsection{Improvements provided by clustered sampling}\label{sec:clustered_sampling_improvments}

We introduced clustered sampling in Section \ref{sec:clustered_sampling_introduction} and showed that under the condition of Proposition \ref{prop:equivalent} it provides the same convergence bound of MD sampling. In this section, we investigate the statistical benefits of clustered sampling with respect to MD sampling.

We define by $\omega_i(S)$ the aggregation weight of client $i$ with subset of sampled clients $S$, and by $S_{MD}$ and $S_C(t)$ the subset of clients sampled at iteration $t$ with respectively MD and clustered sampling.

We consider a clustered sampling scheme following Proposition \ref{prop:equivalent}. Hence, for both MD and clustered sampling, the expected aggregation equals the deterministic aggregation when considering all the clients leading to:
\begin{equation}
\EE{S_{MD}(t)}{\omega_i(S_{MD})}
= \EE{S_C(t)}{\omega_i(S_C(t))} 
= p_i 
. 
\end{equation}
With clustered sampling, we first show that every client has a smaller aggregation weight variance. 
A client's aggregation weight can be written as $\omega_i(S) = \frac{1}{m}\sum_{k=1}^{m}\{ l_k = i\}$, where $l_k$ is the index of the $k^\text{th}$ sampled client. With MD sampling, $m$ clients are iid sampled according to $\mathcal{B}(p_i)$, a Bernoulli distribution with probability $p_i$, giving the following variance:
\begin{align}
\VARR{S_{MD}}{\omega_i(S_{MD})}
&= \frac{1}{m^2}m \VAR{\mathcal{B}(p_i)}\\
&= \frac{1}{m^2}m p_i(1-p_i) .
\end{align}
Clustered sampling instead selects independently $m$ clients according to the distributions $\{W_k(t)\}_{k=1}^m$. Therefore, each client is sampled according to $\mathcal{B}(r_{k, i}^t)$ giving:
\begin{align}
\VARR{S_{C}(t)}{\omega_i(S_C(t))}
&= \frac{1}{m^2} \sum_{k=1}^{m} \VAR{\mathcal{B}(r_{k, i}^t)}\\
&= \frac{1}{m^2} \sum_{k=1}^{m} r_{k,i}^t(1-r_{k,i}^t) \label{eq:variance_clustered}
.
\end{align}
By the Cauchy-Schwartz inequality, one can prove that 
\begin{equation}
\VARR{S_{MD}}{\omega_i(S_{MD})}
\ge \VARR{S_{C}(t)}{\omega_i(S_C(t))},
\end{equation}
with equality if and only if all the $m$ distributions are equal to the one of MD sampling, i.e. $\forall k,\ W_k(t) = W_0$. Exact derivation is given in Appendix \ref{app:stat_measure_clust_2}. Therefore, with clustered sampling, every client has a smaller aggregation weight variance. Another interesting statistical measure of representativity is the probability for a client to be sampled, i.e. $\mathbf{P}(\{i\in S\})$. 
In particular, increasing the probability for every client to be sampled is mandatory to allow a proper representation of each client's data specificity in the global model, especially in heterogeneous setting, such as with non-iid and unbalanced clients.
For this statistical measure, clustered sampling also provides better guarantees than MD sampling. With MD sampling, clients are iid sampled giving
\begin{align}
p(i \in S_{MD})
&= 1 - p(\{i \notin S_{MD}\})\\
&= 1 - p(\{i \notin W_0\})^m\\
&= 1 - (1-p_i)^m .
\end{align}
Similarly, with clustered sampling we get:
\begin{align}
p(i \in S_C(t))
&= 1 - \prod_{k=1}^m p(\{i \notin W_k(t)\})\\
&= 1 - \prod_{k=1}^m (1- r_{k,i}^t) .
\end{align} 
Since we assume here that clustered sampling follows 
Proposition \ref{prop:equivalent}, from equation (\ref{eq:total_probability}), and from the inequality of arithmetic and geometric means, we get:
\begin{equation}
p(\{i \in S_{MD}(t)\}) 
\le
p(\{i \in S_C(t)\}),
\end{equation}
with equality if and only if all the $m$ distributions are equal to the one of MD sampling, i.e. $\forall k,\ W_k(t) = W_0$ (derivation in Appendix \ref{app:stat_measure_clust_2}). Therefore, with clustered sampling, every client has an higher probability of being sampled and thus is better represented throughout the FL process.

In conclusion, clustered sampling reduces clients aggregation weights variance and increases their representativity. These results are important for FL applications with heterogeneous federated dataset. Increasing a client representativity ensures that clients with unique distributions are more likely of being sampled, and can potentially lead to smoother and faster FL convergence.

\section{Clustered sampling based on sample size}\label{sec:basic_clustered}

\begin{algorithm}[tb]
	\caption{Clustered sampling based on sample size}
	\label{alg:clustered_basic}
	\begin{algorithmic}[1]
		\STATE {\bfseries Input:} $\{n_i\}_{i=1}^n$ clients number of samples
		
		\STATE Order clients by descending importance of $n_i$.
		
		\STATE $k \leftarrow 1$ distribution index.
		
		\STATE $q \leftarrow 0$ sum of samples.
		
		\STATE $M \leftarrow \sum_{i=1}^n n_i$ total number of samples.

		\FOR{each client $i=1$ {\bfseries to} $n$}
		
		    \STATE $q \leftarrow q + m n_i$
		    
		    \STATE $q = a_i M + b_i$ with $a_i$ and $b_i$ non negative integers
		    
		    \IF{$a_i > k$}
		    
				\STATE $r_{k, i}' \leftarrow M - b_{i-1}$
				
				\STATE $\forall l \ge k +1 \text{ s.t. } (a_i -1) - l\ge 0,\ r_{k, i}' \leftarrow M$
			\ENDIF
			
			\STATE $r_{a_i, i}' \leftarrow b_i$
			
			\STATE $ k \leftarrow a_i$

		\ENDFOR
		
		\STATE {\bfseries Output:}
        Sampling probabilities $r_{k, i} = r_{k, i}'/M$.
		
	\end{algorithmic}
\end{algorithm}

We introduced and showed the convergence of unbiased clustered sampling in Section \ref{sec:clustered_sampling}. Clustered sampling schemes compatible with Proposition \ref{prop:equivalent} are numerous, including MD sampling. In this section, we first provide an unbiased clustered sampling scheme based on the number of samples $n_i$ owned by each client.
The proposed scheme, illustrated in Algorithm \ref{alg:clustered_basic} , is compatible with Proposition \ref{prop:equivalent}. In particular, we have the following theorem:

\begin{theorem}\label{theo:basic_clustered}
	Algorithm \ref{alg:clustered_basic} outputs $m$ distributions for a clustered sampling satisfying Proposition \ref{prop:equivalent}. The complexity of the algorithm is $\Ocal(n log(n))$.
\end{theorem}


\begin{proof} 
Algorithm \ref{alg:clustered_basic} identifies the $m$ distributions $W_k$ by defining $m$ sets $q_k$ of cardinality $M$ in which each client $i$ is represented with probability $r_{k,i}$. The sets are constructed as follows. We define by $n_i’ =  m n_i$ the total number of samples to be allocated for each client. We thus have $mM$ samples to allocate over the $m$ sets $q_k$. For each client, the integer division $n_i’ = M a_i + b_i$, identifies $a_i$ sets for which the client must be represented with probability 1. The remaining $b_i$ samples are allocated to the remaining $m -\sum_i a_i$ sets. This is possible by observing that $Mm = \sum_i n_i' = M (\sum_i a_i) + \sum_i b_i$, and therefore $M (m -\sum_i a_i) = \sum_i b_i$. By construction, Proposition \ref{prop:equivalent} is satisfied: $|q_k| = M$ implies equation (\ref{eq:sum_proba}), while equation (\ref{eq:total_probability}) is met since, for each client, the total number of samples distributed across the sets $q_k$ is $n_i’ = m n_i$, and thus each client is represented with proportion $m p_i$ across all the distributions.
Algorithm \ref{alg:clustered_basic} provides the practical implementation of this scheme. 

The complexity of Algorithm \ref{alg:clustered_basic} is derived in Appendix \ref{app:algs_proofs}, where we also provide a schematic illustration of the allocation procedure.
\end{proof}


We note that with Algorithm \ref{alg:clustered_basic} a client $i$ can be sampled up to $\floor{m p_i} +2$ times. This is an improvement from MD sampling where clients can be instead sampled up to $m$ times. Since clustering is performed according to the clients sample size $n_i$, unless $n_i$ changes during the learning process, Algorithm \ref{alg:clustered_basic} needs to be run only once at the beginning of the learning process, i.e. $\{W_k(t)\}_{k=1}^m = \{W_k\}_{k=1}^m$.

\section{Clustered sampling based on similarity}\label{sec:clustered_improved}

\begin{algorithm}[tb]
	\begin{algorithmic}[1]
		
		\STATE {\bfseries Input:} $\{n_i\}_{i=1}^n$ clients number of samples, $\{G_i\}_{i=1}^n$ clients representative gradient, $m$ number of sampled clients, \textit{clustering method} (e.g. Ward method), $s$ similarity function (e.g. Arccos)
		
		\STATE Estimated hierarchical clustering $P$ with \textit{clustering method} from similarity matrix $\rho$ with $\rho_{i,j} = s(G_i, G_j)$.
		
		\STATE Cut $P$ to determine $K\ge m$ groups $\{B_k\}_{k=1}^K$. We define $q_k$ as the total number of samples of the corresponding clients: $q_k = \sum_{i \in B_k} mn_i\le M$.
		
		\STATE Order the groups $\{B_k\}_{k=1}^K$ by decreasing $q_k$.
		
		\STATE Define clients number of samples in the $m$ distributions $\{W_k\}_{k=1}^m$ based on the ranking of $q_k$ : $\forall k \le m, \forall i \in B_k,\ r_{k, i}' \leftarrow m n_i$. 

        \STATE Create a set with the clients of the remaining groups $S = \{\{i, u_i = m n_i\},\ \forall i \in B_{m+1} \cup ... \cup B_K\}$
		
		\STATE $k \leftarrow 1$ Start considering the first distribution $W_k$
		
		\REPEAT
		
			\STATE Select first client $i$ in $S$ with $u_i$ samples to allocate
			
			
			\STATE Determine $a_i$ and $b_i$ the quotient and remainder of the euclidean division of $q_k + u_i$ by $M$

			\IF{$a_i = 0$}
				\STATE $r_{k, i}' \leftarrow b_i$ and $i$ removed from $S$
				
				
				
			\ELSE
				\STATE $r_{k, i}' \leftarrow M - q_k$ 
				\STATE $u_i \leftarrow u_i - r_{k, i}'$ 
                \STATE Remove $i$ from $S$ if $u_i = 0$
				\STATE $k \leftarrow k + 1$

			\ENDIF
			
		\UNTIL{$S = \emptyset $, Proposition \ref{prop:equivalent} is satisfied}
		
		\STATE {\bfseries Output:} Sampling probabilities $r_{k,i} = r_{k,i}'/M$.
	\end{algorithmic}
	\caption{Clustered sampling based on model similarity}
	\label{alg:improved_clus_sampling}
\end{algorithm}

We have shown, in Section \ref{sec:clustered_sampling}, that unbiased clustered sampling is a generalization of MD sampling providing smaller aggregation weight variance for every client, and we proposed in Section \ref{sec:basic_clustered} an algorithm to practically fulfill Proposition \ref{prop:equivalent} to obtain $m$ distributions grouping clients based on their number of samples $n_i$. 

In this section we extend the approach of Section \ref{sec:clustered_sampling} to define a novel clustered sampling scheme where sampling distributions are defined based on the similarity across clients.
In what follows we define clients similarity based on the measure of \textit{representative gradient}. The representative gradient is the difference between a client's updated model and the global model. Comparing clients' representative gradients at a given iteration is shown to be an effective approach for detecting similarity between FL participants \cite{CFL}. 

Algorithm \ref{alg:improved_clus_sampling} adopts this concept to define a clustered sampling scheme compatible with Proposition \ref{prop:equivalent}. We have:


\begin{theorem}\label{theo:improved_clustered}
    If for every client $p_i\le 1/m$, Algorithm \ref{alg:improved_clus_sampling} outputs $m$ distributions for a clustered sampling satisfying Proposition \ref{prop:equivalent}. The complexity of the algorithm is in $\Ocal(n^2d + X)$, where $d$ is the number of parameters in the model, and $X$ is the complexity of the clustering method.
\end{theorem}


\begin{proof} Algorithm \ref{alg:improved_clus_sampling} is similar to Algorithm \ref{alg:clustered_basic}, with the additional constraint that the number of clusters $K$ can differ from the number $m$ of distributions. If $K = m$, the clients are already allocated in $q_k$ sets, and the same reasoning of Algorithm \ref{alg:clustered_basic} can be applied. If $K > m$, we consider again the partitioning problem over $m$ sets $q_k$ of cardinality $M$. We define again by $n_i’ =  m n_i$ the total number of samples to be allocated for each client, and we have $mM$ samples to allocate over the $m$ sets $q_k$. Differently from Algorithm \ref{alg:clustered_basic}, we initialize the allocation with the clustering. In particular, we assign to each set $q_k$ the $n_i’$ samples of the clients included in cluster $k$. By construction, each of these sets $q_k$ has cardinality  $|q_k|\leq M$. We consider the $m$ largest sets, and distribute in these sets the remaining samples of the $K-m$ clusters until $|q_k|=M$, for each $k$. By construction, this allocation is possible since we have $mM$ total number of samples to be distributed across $m$ sets of cardinality $M$.
As for Algorithm \ref{alg:clustered_basic}, Proposition \ref{prop:equivalent} is satisfied: $|q_k| = M$ implies equation (\ref{eq:sum_proba}), while equation (\ref{eq:total_probability}) is met since, for each client, the total number of samples distributed across sets $q_k$ is $n_i’ = m n_i$. 
Appendix \ref{app:algs_proofs} completes the proof on the complexity of the algorithm.
\end{proof}
As for Algorithm \ref{alg:clustered_basic}, Appendix \ref{app:algs_proofs} provides a schematic for a better illustration of the algorithm. Being a clustered sampling scheme, the variance of the clients aggregation weights of Algorithm \ref{alg:improved_clus_sampling} is bounded (equation (\ref{eq:variance_clustered})). Moreover, since the  distributions are obtained from the similarity tree resulting from the representative gradients, this scheme explicitly promotes the sampling of clients based on their similarity. Finally, with Algorithm \ref{alg:improved_clus_sampling}, the sampling from the distributions $\{W_k\}_{k=1}^m$ can be performed even when no representative gradient is available for the clients, for example if clients have not been sampled during FL yet. In this case we simply consider a constant 0 representative gradient for those clients, and thus group them together to promote their representativity in the same distribution. 

We recall that Algorithm \ref{alg:improved_clus_sampling} does not require to share gradients across clients, but only the difference between local and global models (a.k.a. representative gradients). Thus, the communication cost is the same of standard FL while the privacy properties of FL privacy remain identical.


We emphasize that any valid hierarchical clustering algorithm can be used in Algorithm \ref{alg:improved_clus_sampling}. Without loss of generality, in the rest of this work we consider the Ward hierarchical clustering method \cite{WARD_theorem}, which allows to obtain a similarity tree by minimizing at every node the variance of its depending clients. This method has complexity $\Ocal(n^2 \log(n))$.
We finally observe that the time complexity of Algorithm \ref{alg:improved_clus_sampling}  is not necessarily an issue, even in presence of an important amount of clients. After aggregation of the new global model, the server can sample the clients, and transmit it to them. While waiting for their local work to be completed, the server can therefore estimate the new partitioning. In this way, Algorithm \ref{alg:improved_clus_sampling} is equivalent to MD sampling for what concerns the process of receiving the updated models, and transmitting the new global model to the clients. 

As a final observation, while Algorithm \ref{alg:improved_clus_sampling} is originally designed for sampling scenarios where $p_i\le 1/m$, with few modifications it  can be also used for federated datasets composed of clients with larger sample size, i.e. when $I = \{i: p_i \ge 1/m \}\neq \emptyset$, or equivalently $I = \{i: m n_i \ge M \}\neq \emptyset$. In this case, 
we can simply allocate those clients in specific distributions, where they are sampled with probability 1. In total, we obtain $\floor{m \frac{n_i}{M}}$ distributions of this kind. The remaining samples  $m n_i - \floor{m \frac{n_i}{M}}M < M$ will be then redistributed according to Algorithm \ref{alg:improved_clus_sampling}. 
\section{Experiments}\label{sec:experiments}

\begin{figure}
	\includegraphics[width=\linewidth]{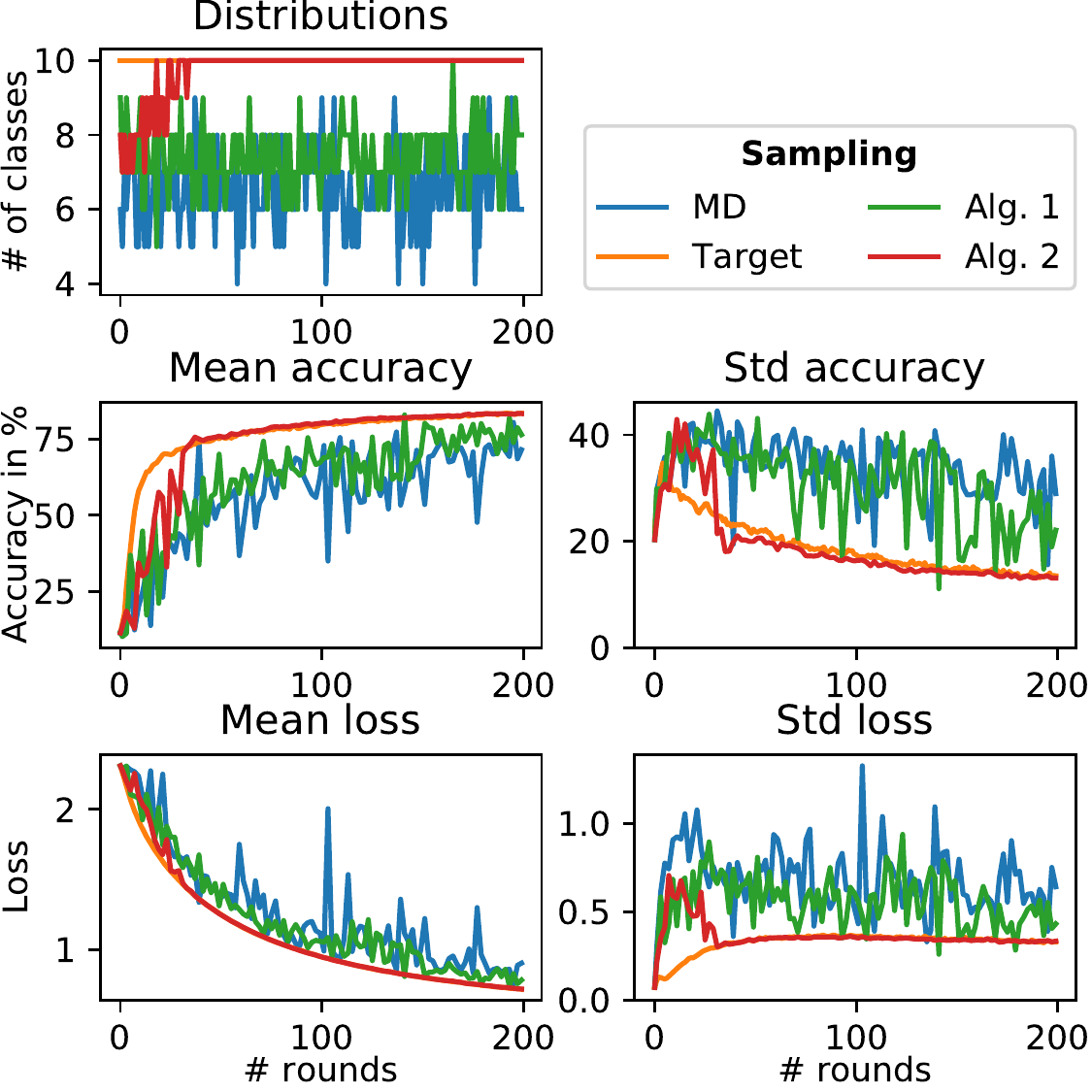}
	\caption{Comparison of MD sampling with clustered sampling of Algorithm \ref{alg:clustered_basic} and \ref{alg:improved_clus_sampling} using cosine angle for the similarity measure. $n=10$ clients from which $m=10$ are sampled to perform $N=50$ SGD with learning rate $lr=0.01$ and batch size $B=50$.}
	\label{fig:illu_improved_clust}
\end{figure}

We first show on a standard classification problem on MNIST \cite{Lecun1998}, the advantages of clustered sampling obtained with Algorithm \ref{alg:clustered_basic} and \ref{alg:improved_clus_sampling} with respect to MD sampling.  
We consider a fully connected network with one hidden layer of 50 nodes. We create a federated dataset composed of 100 clients where each one has 500 training and 100 testing samples composed by one digit only. Each digit is owned by 10 clients, every client has the same number of samples, and the server samples 10 clients at every iteration.

We note that an ideal clustering method for this FL problem consists in creating 10 clusters each containing the 10 clients with same classes. At each FL round, we should sample a client from each cluster in order to obtain a fair representation of all the digits in the model aggregations. 
We call `target' sampling this ideal FL scenario. In practice, the server cannot adopt `target' sampling as this requires to know the clients data distributions in advance. 
As we shall see in the rest of this section, the controlled nature of this example allows to clearly appreciate the practical benefits of clustered sampling.

We first show in Figure \ref{fig:illu_improved_clust} that the FL processes obtained with Algorithm  \ref{alg:clustered_basic} and \ref{alg:improved_clus_sampling} both outperform MD sampling in terms of training global loss, testing accuracy, and representativity of the sampled classes. Moreover, we note that Algorithm \ref{alg:improved_clus_sampling} converges to the same ideal performances of `target' sampling.

We also note that, with MD sampling, between 6 and 8 clients with different digits are generally chosen at each iteration round (Figure \ref{fig:illu_improved_clust}, top left panel). This is a practical demonstration of the sub-optimal representation of the clients heterogeneity.
From a statistical perspective, with MD sampling the probability of sampling 10 different clients is $p= \frac{100!}{90!100^10}\sim 63\%$. Thus, for 37\% of FL iterations, the new global model results from aggregation of less than 10 distinct clients. 
On the contrary, clustered sampling guarantees by construction that the aggregation will be always performed on 10 different clients. Indeed, since the dataset is balanced and the number of sampled clients $m=10$ is a divider of the number of clients $n=100$, every client can be allocated to one distribution only, and can be thus sampled up to once. Moreover, clustered sampling ensures that all the clients have identical aggregation weight variance.
This improved data representation translates in less convergence variability at every iteration. Figure \ref{fig:illu_improved_clust} illustrates this result by showing noticeable improvements in terms of convergence with lower variance and better performance for training loss and testing accuracy. 
Moreover, with Algorithm \ref{alg:improved_clus_sampling}, although at the early training steps some classes are not represented, the clustering strategy allows to quickly partition the 100 clients in 10 clusters and converge to the ideal distribution of `target' (Figure \ref{fig:illu_improved_clust}, top left). 
As a consequence of this improved representativity of clients and classes, Algorithm \ref{alg:improved_clus_sampling} is associated with smoother and faster convergence processes for training loss and testing throughout iterations. 


To demonstrate the benefits of clustered sampling beyond the controlled setting of MNIST, we conduct additional experiments on CIFAR10 \cite{CIFAR-10} to investigate clustered sampling on more complex data distributions and models. CIFAR10 is composed of 32x32 images with three RGB channels of 10 different classes with 60000 samples. We use the same classifier of \cite{FedAvg} composed of 3 convolutional layers and 2 fully connected ones, including dropout after every convolutional layer.

To measure the influence of non-iid data distributions on the effectiveness of clustered sampling, we partition CIFAR10 using a Dirichlet distribution, $Dir(\alpha)$, giving to each client the respective partitioning across classes. 
The parameter $\alpha$ monitors the heterogeneity of the created dataset: $\alpha=0$ assigns one class only to every client, while $\alpha\rightarrow + \infty$ gives a uniform partitioning of classes to each client.
\cite{FL_and_CIFAR_dir} provides graphical illustration of datasets obtained with such a process, and we provide in Appendix \ref{app:experiments} similar illustrations for the parameters $\alpha\in \{0.001, 0.01, 0.1, 10\}$ considered in this work. To create an unbalanced federated dataset, we consider 100 clients where 10, 30, 30, 20 and 10 clients have respectively 100, 250, 500, 750, and 1000 training samples, and testing samples amounting to a fifth of their training size. All the clients consider a batch size of 50. For every CIFAR10 dataset partition, we report in this work experiments with learning rate in $\{0.001, 0.005, 0.01, 0.05, 0.1\}$ minimizing FedAvg with MD sampling training loss at the end of the learning process.

\begin{figure}
	\includegraphics[width=\linewidth]{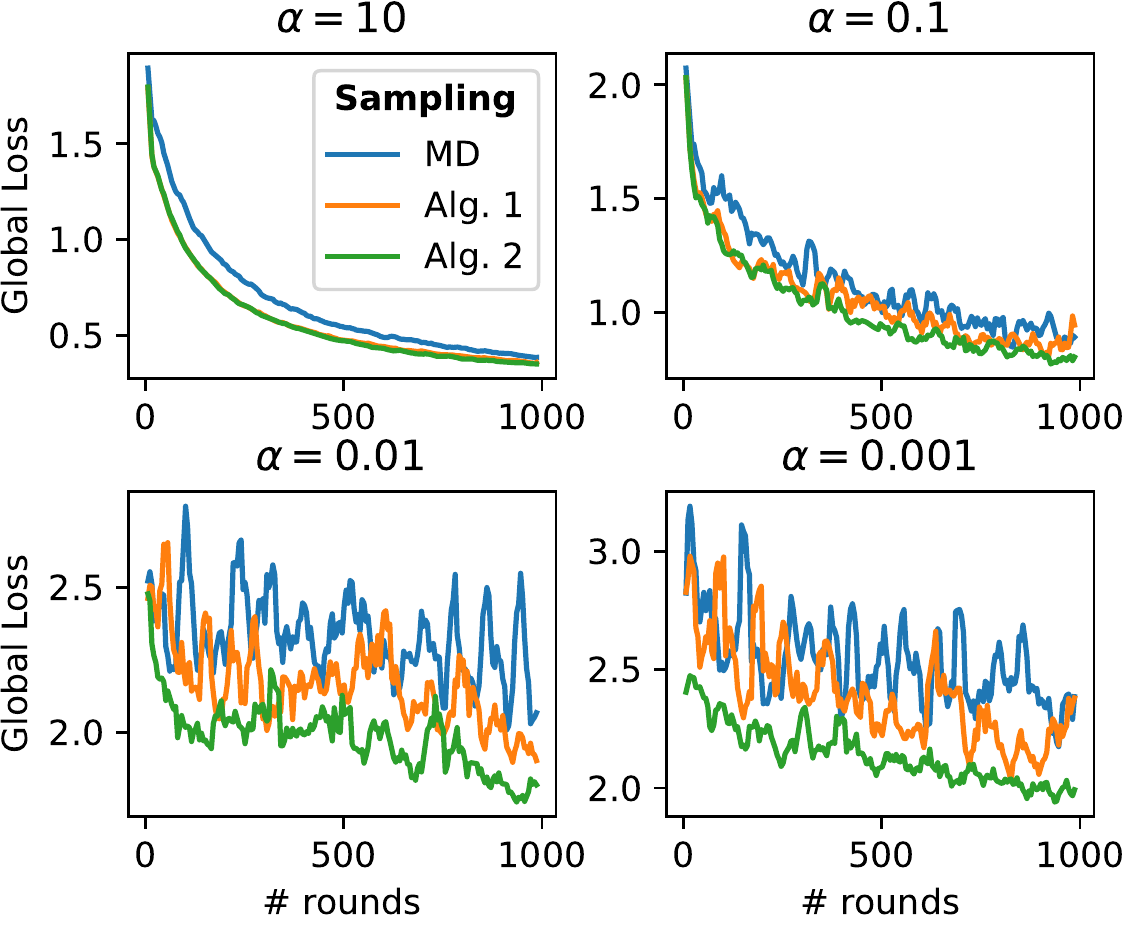}
	\caption{ We investigate the improvement provided by clustered sampling on federated unbalanced datasets partitioned from CIFAR10 using a Dirichlet distribution with parameter $\alpha\in \{0.001, 0.01, 0.1, 10\}$. We use $N=100$, $m=10$, and respective learning rate for each dataset $lr=\{0.05, 0.05, 0.05, 0.1\}$.}
	\label{fig:CIFAR_unbalanced}
\end{figure}

In Figure \ref{fig:CIFAR_unbalanced}, we show how heterogeneity determines the improvements of clustered sampling over MD sampling for any of the four datasets. We note that the more heterogeneous a dataset is, i.e. the smaller $\alpha$, the larger is the improvement provided by clustered sampling. Theorem \ref{theo:improved_clustered} shows that clustered sampling has an identical bound as MD sampling. This is retrieved for $\alpha=10.$ and $\alpha =0.1$ where the final performances for the two samplings are close with faster convergence for clustered sampling. However, with $\alpha=0.01$ and $\alpha = 0.001$, clustered sampling provides faster and better convergence. Overall, with clustered sampling, the evolution of the training loss and testing accuracy are smoother processes than with MD sampling.

For sake of clarity, we note that the training losses displayed in Figure \ref{fig:CIFAR_unbalanced} is computed as the rolling mean over 50 iterations, while we provide in Appendix \ref{app:experiments} the original training loss evolution.
Furthermore, Appendix \ref{app:experiments} reports a larger panel of experiments providing additional verification of the improvements brought by clustered sampling. In Figure \ref{fig:illu_improved_clust} and \ref{fig:CIFAR_unbalanced}, Algorithm \ref{alg:improved_clus_sampling} is computed with Arccos similarity. We show in Appendix \ref{app:experiments} that with L2 and L1 we get similar improvements. We also show that increasing the amount of local work $N$ enables clients to update models fitting better their data distribution. As a result, measuring clients similarity is easier, enabling better clustering, and leading to better performances. We also show that for any amount of sampled clients clustered sampling improves MD sampling. Finally, it is worth noticing that in none of the experimental settings considered for this paper clustered sampling underperformed with respect to MD sampling, providing further experimental evidence for our theoretical results.

\section{Discussion and conclusion}
In this work, we introduced clustered sampling, a novel client selection scheme in FL generalizing MD sampling, the current scheme from the state-of-the-art. We proved the correctness of clustered sampling and proposed two clustering methods implementing aggregation based on the clients number of samples, in Algorithm \ref{alg:clustered_basic}, or model similarity, in Algorithm \ref{alg:improved_clus_sampling}. Both algorithms provide smaller weight variance for the clients aggregation process leading to better client representativity. Consistently, clustered sampling is experimentally shown to have faster and smoother convergence in heterogeneous dataset.

The generality of clustered sampling paves the way to further investigation of clients clustering methods based on different criteria than clients sample size or model similarity. To the best of our knowledge, this work is also the first one introducing model similarity detection when sampling clients, as opposed to current approaches considering all clients at every iteration. 

Finally, clustered sampling is an unbiased sampling scheme simple to implement, while not requiring to modify neither server nor clients behavior during FL training. This aspects makes clustered sampling readily compatible with existing methods and technologies for privacy enhancement and communication reduction.

\subsection*{Acknowledgments and Disclosure of Funding}
\label{sec:ack}

This work has been supported by the French government, through the 3IA Côte d’Azur Investments in the Future project managed by the National Research Agency (ANR) with the reference number ANR-19-P3IA-0002, and by the ANR JCJC project Fed-BioMed 19-CE45-0006-01. The project was also supported by Accenture.
The authors are grateful to the OPAL infrastructure from Université Côte d'Azur for providing resources and support.

\typeout{}
 \bibliography{main.bbl}
\bibliographystyle{./icml2020}

\newpage
\appendix

\section{Proof of Theorem \ref{theo:convergence_clustered}}
\label{app:proof_theo_2}

Theoretical guarantees regarding the convergence of FedAvg were given in \cite{FedNova}. The proof relies on Assumptions \ref{ass:smoothness} to \ref{ass:dissimilarity}.
The full proof is provided in \cite{FedNova} for MD sampling where the MD sampling is shown to satisfy Lemma \ref{lem:sampling_conditions}. In Section \ref{proof:FedNove}, we reproduce the proof provided in \cite{FedNova} for Lemma \ref{lem:sampling_conditions} and, in Section \ref{proof:clustered}, we show that clustered sampling satisfying Proposition \ref{prop:equivalent} also satisfies Lemma \ref{lem:sampling_conditions}. As a result, FedAvg when sampling clients with MD or clustered sampling has identical asymptotic behavior. 

\begin{lemma}\label{lem:sampling_conditions}
	Suppose we are given $z_1, z_2, ..., z_n, x \in \mathbb{R}^d$. Let $l_1, l_2, ..., l_m$ be the index of the sampled clients and $S$ be the set of sampled clients. We have 
	\begin{align}\label{eq:proof_FN_1}
		\EE{S}{\frac{1}{m}\sum_{j=1}^{m}z_{l_j}}
		= \sum_{i=1}^{n}p_i z_i ,
	\end{align}
	and
	\begin{align}\label{eq:proof_FN_2}
		\EE{S}{\norm{\frac{1}{m}\sum_{j=1}^{m}z_{l_j}}^2}
		\le 3\sum_{i=1}^{n}p_i \norm{z_i - \nabla \Lcal_i(x)}^2 \nonumber\\ 
		+ 3 \norm{\nabla \Lcal(x)}^2 
		+ \frac{3}{m}(\beta^2 \norm{\nabla \Lcal(x)}^2 + \kappa^2) .
	\end{align}
\end{lemma}


\subsection{Proof of Lemma \ref{lem:sampling_conditions} for Theorem \ref{theo:convergence_FP} adapted from \cite{FedNova}}\label{proof:FedNove}
\begin{proof}
    Clients are selected with MD sampling. 
    We denote by $l_1, l_2, ..., l_m$ the $m$ indices of the sampled clients which are iid sampled from a multinomial distribution supported on $\{1, ..., n\}$ satisfying $\mathbb{P}(l_x = i) = p_i$ and $\sum_{i=1}^{n}p_i = 1$.
	
	By definition, MD sampling satisfies equation (\ref{eq:proof_FN_1}).
	
	Regarding equation (\ref{eq:proof_FN_2}), we have:
	\begin{align}
	\frac{1}{m}\sum_{j=1}^{m}z_{l_j}
	&= \left(\frac{1}{m}\sum_{j=1}^{m}z_{l_j} - \frac{1}{m}\sum_{j=1}\nabla \Lcal_{l_j}(x)\right) \nonumber\\
	&+ \left(\frac{1}{m}\sum_{j=1}^{m} \Lcal_{l_j}(x) - \nabla \Lcal(x)\right) + \nabla \Lcal(x) .
	\end{align}
	Using the Jensen inequality on the $\norm{\cdot}^2$ operator, we get: 
	\begin{align}
	\E{\norm{\frac{1}{m}\sum_{j=1}^{m}z_{l_j}}^2}
	&\le 3\E{\norm{\frac{1}{m}\sum_{j=1}^{m}(z_{l_j} - \nabla \Lcal_{l_j}(x))}^2} \nonumber\\
	&+ 3\E{\norm{\frac{1}{m}\sum_{j=1}^{m} \nabla \Lcal_{l_j}(x) - \nabla \Lcal(x)}^2} \nonumber\\
	& + 3 \norm{\nabla \Lcal(x)}^2 \label{app:eq:FedNova_three_terms}
	\end{align}	
	Using the Jensen inequality, we get the following upper bound for the first term:
	\begin{align}
		&\E{\norm{\frac{1}{m}\sum_{j=1}^{m}(z_{l_j} - \nabla \Lcal_{l_j}(x))}^2} \nonumber\\
		&\le \E{\frac{1}{m}\sum_{j=1}^{m}\norm{z_{l_j} - \nabla \Lcal_{l_j}(x)}^2}\\
		&= \sum_{i=1}^{n} p_i \norm{z_i - \nabla \Lcal_i(x)}^2, \label{app:eq:FedNova_first_term}
	\end{align}
	where the equality follows from equation (\ref{eq:proof_FN_1}).
	
	By definition, MD sampling is unbiased, i.e. $\E{\nabla \Lcal_{l_j} (x)} = \nabla \Lcal(x)$. Therefore, we get the following upper bound for the second term:
	\begin{align}
		&\E{\norm{\frac{1}{m}\sum_{j=1}^{m} \nabla \Lcal_{l_j}(x) - \nabla \Lcal(x)}^2} \nonumber\\
		&= \E{\frac{1}{m^2}\sum_{j=1}^{m}\norm{ \nabla \Lcal_{l_j}(x) - \nabla \Lcal(x)}^2}\\
		&= \frac{1}{m}\sum_{i=1}^{n}  p_i {\norm{\nabla \Lcal_i(x) - \nabla \Lcal(x)}^2}\\
		&= \frac{1}{m}\sum_{i=1}^{n}  p_i \norm{\nabla \Lcal_i(x)}^2 - \frac{1}{m}\norm{\nabla \Lcal(x)}^2 \label{app:eq:MD_equality}\\
		&\le \frac{1}{m}[(\beta^2 - 1)\norm{\nabla \Lcal(x)}^2+ \kappa^2]\\
		&\le \frac{1}{m}[\beta^2 \norm{\nabla \Lcal(x)}^2+ \kappa^2] , \label{app:eq:FedNova_second_term}
	\end{align}
	where the first inequality comes from using Assumption \ref{ass:dissimilarity}.
	
	Finally, substituting equation (\ref{app:eq:FedNova_first_term}) and (\ref{app:eq:FedNova_second_term}) in equation (\ref{app:eq:FedNova_three_terms}) completes the proof.
	
\end{proof}

\subsection{ Proof of Lemma \ref{lem:sampling_conditions} for Theorem \ref{theo:convergence_clustered}}\label{proof:clustered}
\begin{proof}
    Clients are selected with clustered sampling. 
	The $m$ clients indices $l_1, l_2, ..., l_m$ are still independently sampled but no longer identically. Each index $l_k$ is sampled from a distribution $W_k$. Each client can be sampled with probability $\mathbb{P}(l_k = i) = r_{k, i}$. 

    Clustered sampling follows Proposition \ref{prop:equivalent} and therefore satisfies equation (\ref{eq:proof_FN_1}).
    
    Equation (\ref{app:eq:FedNova_three_terms}) holds for any sampling schemes. Therefore, we also use it to prove equation (\ref{eq:proof_FN_2}) for clustered sampling. 
    Using the same steps as for the proof of Lemma \ref{lem:sampling_conditions} for MD sampling, we bound the first term of equation (\ref{app:eq:FedNova_three_terms}) as:
	\begin{align}
	&\E{\norm{\frac{1}{m}\sum_{j=1}^{m}(z_{l_j} - \nabla \Lcal_{l_j}(x))}^2} \nonumber\\
	&\le \sum_{i=1}^{n} p_i \norm{z_i - \nabla \Lcal_i(x)}^2.	\label{app:eq:ours_first_term}
	\end{align}
	Before bounding the second term, we define $\nabla \Lcal_{W_k}(x)$ as the expected gradient of the distribution $W_k$ with respects to the parameters $x$, i.e.
	\begin{equation}\label{eq:def_expected_grad_distri}
	\nabla \Lcal_{W_k}(x)
	\coloneqq \EE{l_k \sim W_k}{\nabla \Lcal_{l_k}(x)}
	= \sum_{i=1}^{n}r_{k,i} \nabla \Lcal_i(x)
	\end{equation}

	Using this definition, we bound the second term as
	\begin{align}
	&\E{\norm{\frac{1}{m}\sum_{k=1}^{m} \nabla \Lcal_{l_k}(x) - \nabla \Lcal(x)}^2} \nonumber\\
	&= \E{\norm{\frac{1}{m}\sum_{k=1}^{m} (\nabla \Lcal_{l_k}(x) - \nabla \Lcal_{W_k}(x))}^2} \\
	&= \frac{1}{m^2}\sum_{k=1}^{m}  \E{\norm{\nabla \Lcal_{l_k}(x) - \nabla \Lcal_{W_k}(x)}^2}\label{app:eq:cos_1}\\
	&= \frac{1}{m^2}\sum_{k=1}^{m}  \sum_{i=1}^{n} r_{k, i}\norm{\nabla \Lcal_i(x) - \nabla \Lcal_{W_k}(x)}^2\\
	&= \frac{1}{m^2}[\sum_{i=1}^{n} m p_i \norm{\nabla \Lcal_i(x)}^2 - \sum_{k=1}^{m}\norm{\nabla \Lcal_{W_k}(x)}^2]\label{app:eq:cos_2}\\
	&\le \frac{1}{m}[\beta^2 \norm{\nabla \Lcal(x)}^2+ \kappa^2],
	\label{app:eq:ours_second_term}
	\end{align}
	where the last inequality comes from using Assumption \ref{ass:dissimilarity} and equation (\ref{app:eq:cos_1}) and (\ref{app:eq:cos_2}) are obtained with equation (\ref{eq:def_expected_grad_distri}).
	
	Finally, substituting equation (\ref{app:eq:ours_first_term}) and (\ref{app:eq:ours_second_term}) in equation (\ref{app:eq:FedNova_three_terms}) completes the proof.
	
\end{proof}

Equation (\ref{app:eq:MD_equality}) and (\ref{app:eq:cos_2}) allow us to theoretically identify the convergence improvement of clustered sampling over MD sampling. 

We define by $B_{MD} 
= \frac{1}{m}\sum_{i=1}^np_i \norm{\nabla \mathcal{L}_i(x)}^2 
- \frac{1}{m} \norm{\nabla \mathcal{L}(x)}^2$, 
equation (\ref{app:eq:MD_equality}), and 
$B_{Cl} 
= \frac{1}{m}\sum_{i=1}^np_i \norm{\nabla \mathcal{L}_i(x)}^2 
- \frac{1}{m^2}\sum_{k=1}^m\norm{\nabla \mathcal{L}_{W_k}(x)}^2$, equation (\ref{app:eq:cos_2}). 
Using the Jensen inequality, we get 
\begin{align}
- \sum_{k=1}^m \frac{1}{m^2} \norm{\nabla \mathcal{L}_{W_k}(x)}^2 
&\le -  \frac{1}{m}\norm{ \sum_{k=1}^m \frac{1}{m} \nabla \mathcal{L}_{W_k}(x)}^2 \\
&= -  \frac{1}{m}\norm{\nabla \mathcal{L}(x)}^2
\end{align}
with equality if and only if $\forall k, l,\ \nabla \Lcal_{W_k}(x) = \nabla \Lcal_{W_l}(x)$. Thus, $B_{Cl} \leq B_{MD} $ with equality if and only if all the clients have the same data distribution or the considered clustered sampling is MD sampling.

\section{MD and clustered sampling comparison}
\label{app:stat_measure_clust_2}

\subsection{Client aggregation weight variance}

As in Section \ref{sec:clustered_sampling}, we denote by $S_{MD}$ and $S_C(t)$ the random variables associated respectively to MD and clustered sampling. Also in Section \ref{sec:clustered_sampling}, we have shown that 
\begin{equation}
\VARR{S_{MD}}{\omega_i(S_{MD})} 
= \frac{1}{m^2}m p_i (1-p_i),
\end{equation}
and 
\begin{equation}
\VARR{S_C(t)}{\omega_i(S_C(t)} 
= \frac{1}{m^2}\sum_{k=1}^{m} r_{k,i}^t (1- r_{k,i}^t)
.
\end{equation}
Hence, we get:
\begin{align}
\VARR{S_{MD}}{\omega_i(S_{MD})}  - \VARR{S_C(t)}{\omega_i(S_C(t)}  \\
= \frac{1}{m^2}[m p_i(1 -p_i) - \sum_{k=1}^{m}r_{k, i}^t(1-r_{k, i}^t)]
\end{align}
We consider an unbiased clustered sampling. Therefore, the sum of probability for client $i$ over the $m$ clusters satisfies $\sum_{k=1}^{m} r_{k, i}^t = m p_i$ giving:
\begin{align}
&\VARR{S_{MD}}{\omega_i(S_{MD})}  - \VARR{S_C(t)}{\omega_i(S_C(t)}\\
& = \frac{1}{m^2}[\sum_{k=1}^{m}{r_{k,i}^t}^2 - m p_i^2 ]
\end{align}
Using the Cauchy-Schwartz inequality, we get: $\sum_{k=1}^{m}{r_{k,i}^t}^2 \times \sum_{k=1}^{m}1^2 \ge \left(\sum_{k=1}^{m}r_{k, i}^t \times 1 \right)^2 = (mp_i)^2$ due to the unbiased aspect of the considered clustered sampling. As such, we get:
\begin{equation}
\VARR{S_{MD}}{\omega_i(S_{MD})}  - \VARR{S_C(t)}{\omega_i(S_C(t)}  
\ge 0,
\end{equation}
with equality if and only if $r_{k, i}^t = p_i$.

\subsection{Probability for a client to be sampled at least once}

In Section \ref{sec:clustered_sampling}, we have shown that 
\begin{equation}
p(\{i \in S_{MD}\})
= 1 - (1-p_i)^m
\end{equation}
and 
\begin{equation}
p(\{i \in S_C(t)\})
= 1 - \prod_{k=1}^m (1- r_{k,i}^t)
.
\end{equation}
Hence, we get:
\begin{align}
&p(\{i \in S_{MD}\}) - p(\{i \in S_C(t)\})\\
&= \prod_{k=1}^m (1- r_{k, i}^t) - (1- p_i)^m 
\end{align}
We consider an unbiased clustered sampling. Therefore, when using the inequality of arithmetic and geometric means, we get:
\begin{equation}
\prod_{k=1}^m (1- r_{k, i}^t)
\le \left(\frac{\sum_{k=1}^{m}(1-r_{k, i}^t)}{m}\right)^m= (1-p_i)^m,
\end{equation}
with equality if and only if $r_{k, i}^t = p_i$. Finally, we get:
\begin{equation}
p(\{i \in S_{MD}\}) 
\ge
p(\{i \in S_C(t)\})
\end{equation}

\section{Explaining Algorithm \ref{alg:clustered_basic} and \ref{alg:improved_clus_sampling}}\label{app:algs_proofs}

Algorithms \ref{alg:clustered_basic} and \ref{alg:improved_clus_sampling} can be written in term of data ratio $p_i$ instead of samples number $n_i$. While in both cases the algorithms would be correct, it turns out to be simpler to work with quantities of samples $n_i = p_i M$ instead which are integers.
Therefore, without loss of generality, we denote by $r_{k, i}'$ the number of samples allocated by client $i$ to distribution $k$. We retrieve the sampling probability of client $i$ in distribution $W_k$ with $r_{k,i} = \frac{r_{k, i}'}{M}$.

Also, without loss of generality, we prove Algorithms \ref{alg:clustered_basic} and \ref{alg:improved_clus_sampling} at iteration $t$ and therefore we use in the proofs $r_{k, i}$ and $W_k$ instead of $r_{k,i}^t$ and $W_k^t$.

\subsection{Algorithm \ref{alg:clustered_basic}}

\begin{figure*}
\newcommand\length{4}
\newcommand\height{1.}
\newcommand\offset{0.5}
\newcommand\xr{ - \length/2}
\newcommand\xs{ \length/3}
\newcommand\xt{3*\length/4}
\newcommand\xu{2*\length/3 +\length}
\begin{center}
\begin{tikzpicture}
	
		\node[coordinate] (A) at (-\length, 0){} ;
		\node[coordinate] (B) at (0, 0){} ;
		\node[coordinate] (C) at (\length, 0){} ;
		\node[coordinate] (D) at (2* \length, 0){} ;
		
		\node[coordinate] (H) at (2 * \length, \height){} ;
		\node[coordinate] (G) at (\length, \height){} ;
		\node[coordinate] (F) at (0, \height){} ;
		\node[coordinate] (E) at (-\length, \height){} ;
		
		\draw (A) rectangle (F) node[midway] {$W_{k-1}$};
		\draw (B) rectangle (G) node[midway] {$W_{k}$};
		\draw (C) rectangle (H) node[midway] {$W_{k+1}$};

		\node[coordinate] (L1) at (-2.5 * \length, \height){} ;
		\node[coordinate] (L2) at ( -1.5 * \length, \height){} ;
		\node[coordinate] (L3) at (-2.5 * \length, 0){} ;
		\node[coordinate] (L4) at (-1.5 * \length, 0){} ;
		
		\draw (L4) rectangle (E) node[midway] {$...$};
		
		\node[coordinate] (R1) at ( 2.5 *\length, \height){} ;
		\node[coordinate] (R2) at (3.5 * \length, \height){} ;
		\node[coordinate] (R3) at (2.5 *\length, 0){} ;
		\node[coordinate] (R4) at (3.5* \length, 0){} ;
		
		\draw (D) rectangle (R1) node[midway] {$...$};

		\node[coordinate] (R) at ( \xr, \height + \offset){} ;
		\node[coordinate] (S) at (\xs,  \height + \offset){} ;
		\node[coordinate] (T) at (\xt,  \height + \offset){} ;
		\node[coordinate] (U) at (\xu,  \height + \offset){} ;
		
		\draw[>=latex,<->] (R) -- (S) node[midway, above]{\footnotesize $m n_{i-1}$};
        \draw[>=latex,<->] (S) -- (T) node[midway,above]{\footnotesize $m n_{i}$};
        \draw[>=latex,<->] (T) -- (U) node[midway,above] {\footnotesize $m n_{i+1}$};
        
		\node[coordinate] (B') at (0, 0-\offset){} ;
		\node[coordinate] (C') at (\length, -\offset){} ;
		\node[coordinate] (R') at ( \xr, - \offset){} ;
		\node[coordinate] (S') at (\xs,  - \offset){} ;
		\node[coordinate] (T') at (\xt,  - \offset){} ;
		\node[coordinate] (U') at (\xu, - \offset){} ;
		
        \draw[>=latex,<->] (R') -- (B') node[midway,below]{\footnotesize $r_{k-1, i-1}$};
        \draw[>=latex,<->] (B') -- (S') node[midway, below]{\footnotesize $r_{k, i-1}$};
        \draw[>=latex,<->] (S') -- (T') node[midway,below]{\footnotesize $r_{k, i}$};
        \draw[>=latex,<->] (T') -- (C') node[midway,below] {\footnotesize $r_{k, i+1}$};
        \draw[>=latex,<->] (C') -- (U') node[midway,below]{\footnotesize $r_{k+1, i+1}$};
        
        \node[coordinate] (A'') at (-\length, - 2 *\offset){} ;
		\node[coordinate] (B'') at (0, - 2 *\offset){} ;
		\node[coordinate] (C'') at (\length, - 2 *\offset){} ;
		\node[coordinate] (D'') at (2* \length, - 2 *\offset){} ;
		
		\draw[>=latex,<->] (A'') -- (B'') node[midway,below]{\footnotesize $M$};
        \draw[>=latex,<->] (B'') -- (C'') node[midway, below]{\footnotesize $M$};
        \draw[>=latex,<->] (C'') -- (D'') node[midway,below]{\footnotesize $M$};
        
        \node[coordinate] (begin) at (-1.5 * \length, \height+ 2*\offset){} ;
		\node[coordinate] (end) at ( 2.5 * \length, \height + 2 * \offset){} ;
		
		\draw[>=latex,<->] (begin) -- (end) node[midway,above]{\footnotesize $\sum_{i=1}^n m\ n_i = m\ M$};

\end{tikzpicture}
\end{center}
\caption{Illustration of the clients allocation scheme of Algorithm \ref{alg:clustered_basic}.  Clients are considered in decreasing importance of their number of samples and always allocate client samples to distributions that already received samples but do not yet have $M$ of them. As a result, after allocating a client, all distributions except at most one have 0 or $M$ samples. Client $i$ is only sampled in $W_k$ because every distribution with index inferior to $k$ are filled with clients of index inferior to $i$, and because there is enough room in $W_k$ to receive all the samples that need to be allocated for client $i$.}\label{app:fig:alg_basic}

\end{figure*}
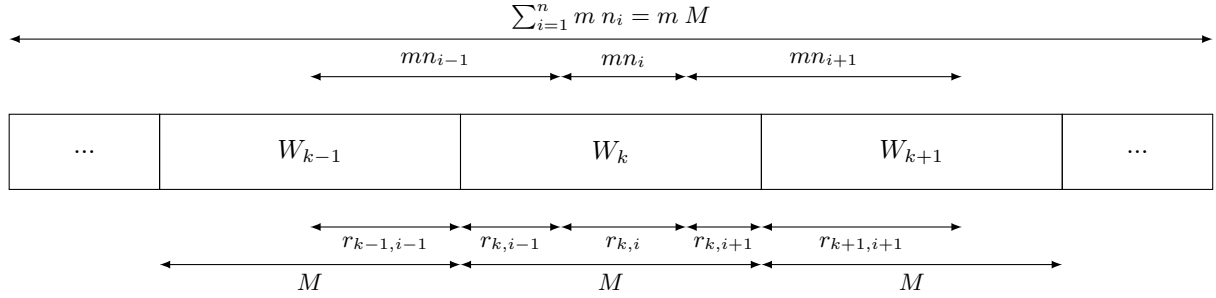

We illustrate in Figure \ref{app:fig:alg_basic} the clients allocation scheme of Algorithm \ref{alg:clustered_basic} introduced in Section \ref{sec:basic_clustered}, by considering how a client $i$ is associated to the $m$ distributions. Theorem \ref{theo:basic_clustered} states that Algorithm \ref{alg:clustered_basic} provides a sampling scheme satisfying Proposition \ref{prop:equivalent} with complexity $\Ocal(n \log(n))$ which we prove in Section \ref{sec:basic_clustered} and in the following proof.

\begin{proof}
    
	
	In term of complexity, the while loop for the client allocation, as illustrated in Figure \ref{app:fig:alg_basic}, either change client or distribution at every step and is thus done in complexity $\Ocal(n+m)$. Sampling client is relevant if $m<n$. Therefore the allocation complexity is equivalent to $\Ocal(n + m) = \Ocal(n)$. Also, sorting $n$ elements is done in complexity $\Ocal(n\log(n))$. Therefore, Algorithm \ref{alg:clustered_basic} overall complexity is $\Ocal(n\log(n)$.
\end{proof}

\subsection{Algorithm \ref{alg:improved_clus_sampling}}

\begin{figure*}
\newcommand\Offset{0.4}
\newcommand\Height{1.}
\newcommand\unit{0.45}
\newcommand\M{6 *\unit}
\newcommand\Between{4 *\Offset}

\newcommand\Xa{ 0}
\newcommand\Xaa{ \Xa + \M}
\newcommand\Xb{ \Xaa + 5 *\unit}
\newcommand\Xc{\Xb + 5 *\unit}
\newcommand\Xd{\Xc + 5 *\unit}
\newcommand\Xdd{\Xd + 4 *\unit}
\newcommand\Xe{ \Xdd + 3*\unit}
\newcommand\Xf{ \Xe + 2 *\unit}
\newcommand\Xg{\Xf +2*\unit}
\newcommand\Xh{\Xg + 2*\unit}
\newcommand\Xj{ \Xh + \unit}

\newcommand\Xs{ \Xa + \M}
\newcommand\Xss{ \Xs + \M}
\newcommand\Xt{\Xss + \M}
\newcommand\Xtt{\Xt + \M}
\newcommand\Xu{\Xtt + 5 *\unit}
\newcommand\Xv{ \Xu + \M}
\begin{center}
\begin{tikzpicture}
	
		\node[coordinate] (A) at (\Xa, \Height+ \Between){} ;
		\node[coordinate] (AA) at (\Xaa, \Height+ \Between){} ;
		\node[coordinate] (B) at (\Xb, \Height+ \Between){} ;
		\node[coordinate] (C) at (\Xc, \Height+ \Between){} ;
		\node[coordinate] (D) at (\Xd, \Height+ \Between){} ;
		\node[coordinate] (DD) at (\Xdd, \Height+ \Between){} ;
		\node[coordinate] (E) at (\Xe, \Height+ \Between){} ;
		\node[coordinate] (F) at (\Xf, \Height+ \Between){} ;
		\node[coordinate] (G) at (\Xg, \Height+ \Between){} ;
		\node[coordinate] (H) at (\Xh, \Height+ \Between){} ;
		\node[coordinate] (I) at (\Xj, \Height+ \Between){} ;
		
		\node[coordinate] (A') at (\Xa, 2 *\Height+ \Between){} ;
		\node[coordinate] (AA') at (\Xaa, 2 *\Height+ \Between){} ;
		\node[coordinate] (B') at (\Xb, 2 *\Height+ \Between){} ;
		\node[coordinate] (C') at (\Xc, 2 *\Height+ \Between){} ;
		\node[coordinate] (D') at (\Xd, 2 *\Height+ \Between){} ;
		\node[coordinate] (DD') at (\Xdd, 2* \Height+ \Between){} ;
		\node[coordinate] (E') at (\Xe, 2 *\Height+ \Between){} ;
		\node[coordinate] (F') at (\Xf, 2 *\Height+ \Between){} ;
		\node[coordinate] (G') at (\Xg, 2 *\Height+ \Between){} ;
		\node[coordinate] (H') at (\Xh, 2 *\Height+ \Between){} ;
		\node[coordinate] (I') at (\Xj, 2 *\Height+ \Between){} ;
		
		\draw (A) rectangle (AA') node[midway] {$B_1$};
		\draw (AA) rectangle (B') node[midway] {$B_2$};
		\draw (B) rectangle (C') node[midway] {$B_3$};
		\draw (C) rectangle (D') node[midway] {$B_4$};
		\draw (D) rectangle (DD') node[midway] {$...$};
		\draw (DD) rectangle (E') node[midway] {$B_m$};
		\draw (E) rectangle (F') node[midway] {$B_{m+1}$};
		\draw (F) rectangle (G') node[midway] {$B_{m+2}$};
		\draw (G) rectangle (H') node[midway] {$...$};
		\draw (H) rectangle (I') node[midway] {$B_K$};

		\node[coordinate] (R) at (\Xa, 0){} ;
		\node[coordinate] (S) at (\Xs, 0){} ;
		\node[coordinate] (SS) at (\Xss, 0){} ;
		\node[coordinate] (T) at (\Xt, 0){} ;
		\node[coordinate] (TT) at (\Xtt, 0){} ;
		\node[coordinate] (U) at (\Xu, 0){} ;
		\node[coordinate] (V) at (\Xv, 0){} ;
		
		\node[coordinate] (R') at (\Xa, \Height){} ;
		\node[coordinate] (S') at (\Xs, \Height){} ;
		\node[coordinate] (SS') at (\Xss, \Height){} ;
		\node[coordinate] (T') at (\Xt, \Height){} ;
		\node[coordinate] (TT') at (\Xtt, \Height){} ;
		\node[coordinate] (U') at (\Xu, \Height){} ;
		\node[coordinate] (V') at (\Xv, \Height){} ;
		
		\draw (R) rectangle (S') node[midway] {$W_1$};
		\draw (S) rectangle (SS') node[midway] {$W_2$};
		\draw (SS) rectangle (T') node[midway] {$W_3$};
		\draw (T) rectangle (TT') node[midway] {$W_4$};
		\draw (TT) rectangle (U') node[midway] {$...$};
		\draw (U) rectangle (V') node[midway] {$W_m$};

		\draw[>=latex,<->] ([yshift = \Offset cm]A') -- ([yshift = \Offset cm]AA') node[midway, above]{\footnotesize $q_1$};
        \draw[>=latex,<->] ([yshift = \Offset cm]AA')  -- ([yshift = \Offset cm]B') node[midway,above]{\footnotesize $q_2$};
        \draw[>=latex,<->] ([yshift = \Offset cm]B') -- ([yshift = \Offset cm]C') node[midway,above] {\footnotesize $q_3$};
        \draw[>=latex,<->] ([yshift = \Offset cm]C') -- ([yshift = \Offset cm]D') node[midway,above] {\footnotesize $q_4$};
        
        \draw[>=latex,<->] ([yshift = \Offset cm]DD') -- ([yshift = \Offset cm]E') node[midway, above]{\footnotesize $q_m$};
        \draw[>=latex,<->] ([yshift = \Offset cm]E')  -- ([yshift = \Offset cm]F') node[midway,above]{\footnotesize $q_{m+1}$};
        \draw[>=latex,<->] ([yshift = \Offset cm]F') -- ([yshift = \Offset cm]G') node[midway,above] {\footnotesize $q_{m+2}$};
        
        \draw[>=latex,<->] ([yshift = \Offset cm]H') -- ([yshift = \Offset cm]I') node[midway, above]{\footnotesize $q_K$};
        
        \node[coordinate] (M1) at (\Xa, -\Offset){} ;
        \node[coordinate] (M2) at (\Xs, -\Offset){} ;
		\node[coordinate] (M3) at (\Xb, -\Offset){} ;
		\node[coordinate] (M4) at (\Xss, -\Offset){} ;
		\node[coordinate] (M5) at (\Xss + 4 * \unit, -\Offset){} ;
		\node[coordinate] (M51) at (\Xss + 5 * \unit, -\Offset){} ;
		\node[coordinate] (M52) at (\Xt, -\Offset){} ;
		\node[coordinate] (M55) at (\Xt, -\Offset){} ;
		\node[coordinate] (M56) at (\Xt + 4 * \unit, -\Offset){} ;
		\node[coordinate] (M57) at (\Xtt, -\Offset){} ;
		\node[coordinate] (M6) at (\Xu, -\Offset){} ;
		\node[coordinate] (M7) at (\Xu + 3 * \unit, -\Offset){} ;
		\node[coordinate] (M71) at (\Xv - \unit, -\Offset){} ;
		\node[coordinate] (M8) at (\Xv , -\Offset){} ;
		
		\draw[>=latex,<->] (M1) -- (M2) node[midway,below]{\footnotesize $q_1$};
		\draw[>=latex,<->] (M2) -- (M3) node[midway,below]{\footnotesize $q_2$};
		\draw[>=latex,<->] (M3) -- (M4) node[midway,below]{\footnotesize $q_{m+1, 1}$};
		\draw[>=latex,<->] (M4) -- (M5) node[midway,below]{\footnotesize $q_3$};
		\draw[>=latex,<->] (M5) -- (M52) node[midway,below]{\footnotesize $q_{m+1, 2}$};
		\draw[>=latex,<->] (M55) -- (M56) node[midway,below]{\footnotesize $q_4$};
		\draw[>=latex,<->] (M56) -- (M57) node[midway,below]{\footnotesize $q_{m+2}$};
		\draw[>=latex,<->] (M6) -- (M7) node[midway,below]{\footnotesize $q_m$};
        \draw[>=latex,<->] (M71) -- (M8) node[midway,below]{\footnotesize $q_K$};

		\node[coordinate] (M1') at (\Xa, -2*\Offset){} ;
        \node[coordinate] (M2') at (\Xs, -2*\Offset){} ;
		\node[coordinate] (M4') at (\Xss, -2*\Offset){} ;
		\node[coordinate] (M5') at (\Xt, -2*\Offset){} ;
		\node[coordinate] (M55') at (\Xtt, -2*\Offset){} ;
		\node[coordinate] (M6') at (\Xu, -2*\Offset){} ;
		\node[coordinate] (M8') at (\Xv, -2*\Offset){} ;
		
		\draw[>=latex,<->] (M1') -- (M2') node[midway,below]{\footnotesize $M$};
		\draw[>=latex,<->] (M2') -- (M4') node[midway,below]{\footnotesize $M$};
		\draw[>=latex,<->] (M4') -- (M5') node[midway,below]{\footnotesize $M$};
		\draw[>=latex,<->] (M5') -- (M55') node[midway,below]{\footnotesize $M$};
		\draw[>=latex,<->] (M6') -- (M8') node[midway,below]{\footnotesize $M$};

		
		
		\node[coordinate] (begin) at (\Xa, 2*\Height+\Between + 2*\Offset){} ;
		\node[coordinate] (end) at ( \Xj, 2*\Height + \Between + 2 * \Offset){} ;
		
		\draw[>=latex,<->] (begin) -- (end) node[midway,above]{\footnotesize $\sum_{l=1}^K q_l = \sum_{l=1}^K \sum_{i\in B_l} r_{l, i}' = \sum_{i=1}^n m n_i = m M$};
		
		\node[coordinate] (Q1) at (3 * \unit, \Height+ \Between){} ;
		\node[coordinate] (Q2) at (8.5 * \unit, \Height+ \Between){} ;
		\node[coordinate] (Q3) at (13.5 * \unit, \Height+ \Between){} ;
		\node[coordinate] (Q4) at (18.5 * \unit, \Height+ \Between){} ;
		\node[coordinate] (QM) at (26.5 * \unit, \Height+ \Between){} ;
		\node[coordinate] (QM1) at (29 * \unit, \Height+ \Between){} ;
		\node[coordinate] (QM2) at (31 * \unit, \Height+ \Between){} ;
		\node[coordinate] (QK) at ( 34.5 * \unit, \Height+ \Between){} ;
		
		\node[coordinate] (V1) at (3 * \unit, \Height){} ;
		\node[coordinate] (V2) at (9 * \unit, \Height){} ;
		\node[coordinate] (V3) at (15 * \unit, \Height){} ;
		\node[coordinate] (V4) at ( 21 * \unit, \Height){} ;
		\node[coordinate] (VM) at (32 * \unit, \Height){} ;
		
		\draw[>=latex,->] (Q1) -- (V1) node[midway,below]{};
		\draw[>=latex,->] (Q2) -- (V2) node[midway,below]{};
		\draw[>=latex,->] (Q3) -- (V3) node[midway,below]{};
		\draw[>=latex,->] (Q4) -- (V4) node[midway,below]{};
		\draw[>=latex,->] (QM) -- (VM) node[midway,below]{};
		
		\draw[>=latex,->] (QM1) -- (V2) node[midway,below]{};
		\draw[>=latex,->] (QM1) -- (V3) node[midway,below]{};
		\draw[>=latex,->] (QM2) -- (V4) node[midway,below]{};
		\draw[>=latex,->] (QK) -- (VM) node[midway,below]{};

\end{tikzpicture}
\end{center}
\caption{Illustration of the clients allocation scheme of Algorithm \ref{alg:improved_clus_sampling}. After the tree is split in $K$ groups of clients, the groups are ordered and we consider without loss of generality that their number of samples are inversely proportional to their index. With Algorithm \ref{alg:improved_clus_sampling}, the first $m$ groups, i.e. $B_1$ to $B_m$, are each associated to one distribution, i.e. $W_1$  to $W_m$. The remaining groups are considered one after the other and split among the remaining slots in the groups. Each distribution has $M$ samples from clients participating to the FL process.}\label{app:fig:alg_similarity}

\end{figure*}
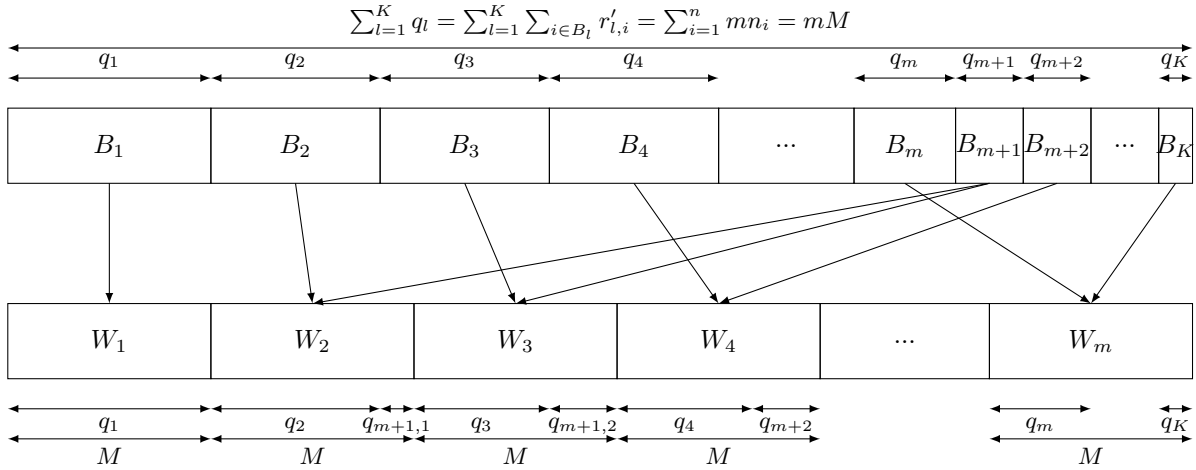

We illustrate in Figure \ref{app:fig:alg_similarity}, the clients allocation scheme of Algorithm \ref{alg:improved_clus_sampling} introduced in Section \ref{sec:clustered_improved} by considering how a client $i$ is associated to the $m$ distributions. Theorem \ref{theo:improved_clustered} states that Algorithm \ref{alg:improved_clus_sampling} provides a sampling scheme satisfying Proposition \ref{prop:equivalent} and takes time complexity $\Ocal(n^2d + X)$. We prove these statements in Section \ref{sec:clustered_improved} and the following proof.

\begin{proof}
    
    With identical reasoning as for Algorithm \ref{alg:clustered_basic}, clients are allocated in complexity $\Ocal(n)$.
	Computing the similarity between two clients requires $d$ elementary operations, where $d$ is the number of parameters in the model, and has thus complexity $\Ocal(d)$. Computing the similarity matrix requires computing $\frac{n(n-1)}{2}$ client similarities and thus has total complexity $\Ocal(n^2 d)$. Computing the similarity tree depends on the \textit{clustering method} which we consider has complexity $\Ocal(X)$. Transforming the tree as discussed in Section \ref{sec:clustered_improved} requires going through its $n-1$ nodes and thus has time complexity $\Ocal(n)$. Cutting the tree requires considering at most every nodes and has thus complexity $\Ocal(n)$. Lastly, the tree is cut in at most $n$ branches and sorting them takes therefore complexity $\Ocal(n\log(n)$. Finally, combining all these time complexities gives for Algorithm \ref{alg:improved_clus_sampling} a time complexity of $\Ocal(n^2 d + X)$.

\end{proof}


In practice, the $m$ distributions are computed at every iteration, while the server is required to compute the similarity between sampled clients and all the other clients. Therefore the similarity matrix can be estimated in complexity $\Ocal(nmd)$, and Algorithm \ref{alg:improved_clus_sampling} has complexity $\Ocal(n m d + X)$.

\section{Additional experiments}\label{app:experiments}

We describe in Section \ref{sec:experiments} the different datasets used for the experiments and how we use the Dirichlet distribution to partition CIFAR10 in realistic heterogeneous federated datasets. In all the experiments, we consider a batch size of 50. For every CIFAR10 dataset partition, the learning rate is selected  in $\{0.001, 0.005, 0.01, 0.05, 0.1\}$ to minimize FedAvg with MD sampling training loss.

\subsection{CIFAR10 partitioning illustration}

In Figure \ref{app:fig:distribution_classes}, we show the influence of $\alpha$ on the resulting federated dataset heterogeneity. $\alpha=10$ provides almost an iid dataset and identical class percentages, column (a) , and same number of samples per class, column (b). With $\alpha = 0.001$, we get a very heterogeneous dataset with almost only one class per client translating into some classes much more represented than others due to the unbalanced nature aspect of the created federated dataset, cf Section \ref{sec:experiments}.
\begin{figure}
	\includegraphics[width=\linewidth]{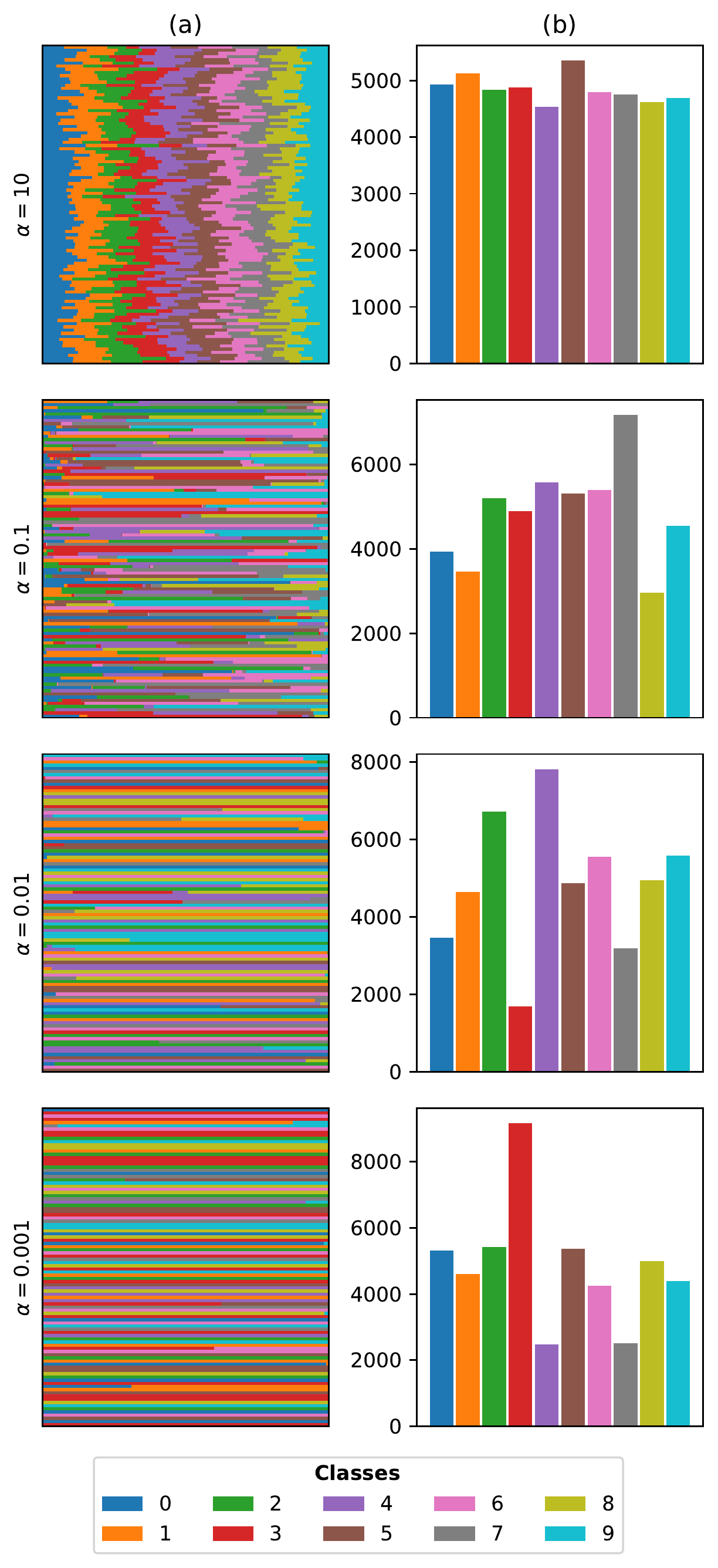}
	\caption{Effect of $\alpha$ on the resulting clients partitioning when using a Dirichlet distribution. Plots in column (a) represent the percentage of each class owned by the clients. Plots in column (b) give for every class its total number of samples across clients. We consider in this work $\alpha \in \{0.001, 0.01, 0.1, 10\}$. }
	\label{app:fig:distribution_classes}
\end{figure}

\subsection{Influence of the similarity measure}

Figure \ref{app:fig:similarity} shows the effect similarity measures (Arccos, L2, and L1) have on training global loss convergence. We retrieve that Algorithm \ref{alg:clustered_basic} outperforms MD sampling by reducing clients aggregation weight variance. We remind that the hierarchical tree is obtained using Ward's method in this work.
We notice that the tree similarity measures gives similar performances when using Algorithm \ref{alg:improved_clus_sampling} with Ward hierarchical clustering method.. This justifies the use of Arccos similarity for the other experiments. 

\begin{figure}
	\includegraphics[width=\linewidth]{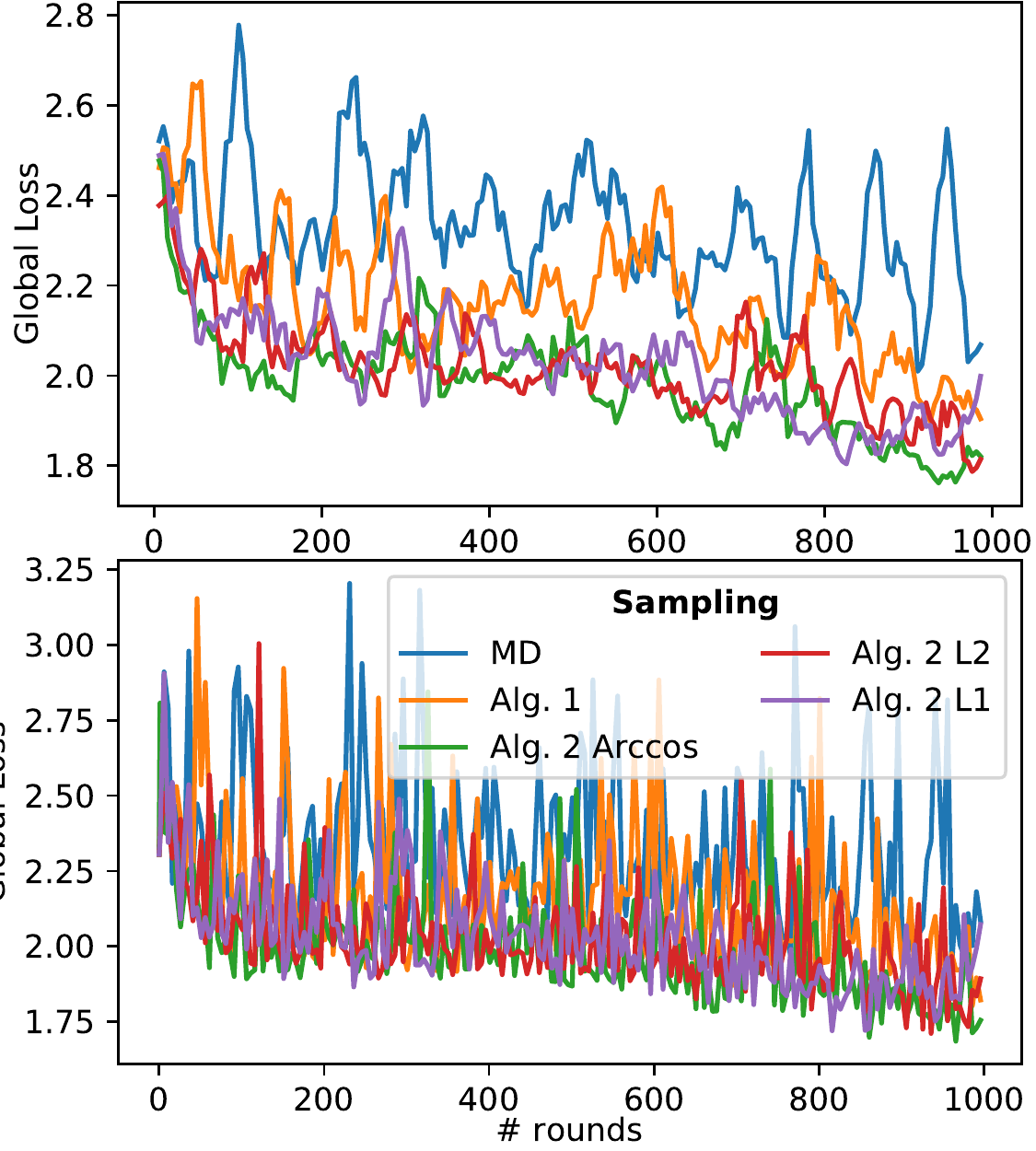}
	\caption{Effect of the similarity measure chosen for Algorithm \ref{alg:improved_clus_sampling} on the training loss convergence. We consider the evolution of the global loss, equation (\ref{eq:global_loss}), in function of the server iteration $t$. For clarity concerns, we plot the global loss obtained with rolling mean over 50 server iterations (top) and the raw global loss (bottom). We consider CIFAR partitioned with Dir($\alpha=0.01$), learning rate $lr=0.05$, $N=100$ SGD, and $m=10$ sampled clients.}
	\label{app:fig:similarity}
\end{figure}

\subsection{More details on Figure \ref{fig:CIFAR_unbalanced}}

For sake of clarity, we note that the training loss displayed in
Figures \ref{fig:CIFAR_unbalanced} is computed as the rolling mean over 50 iterations. In Figure \ref{app:fig:unbalanced_full}, we provide the raw training global loss with the testing accuracy at every server iteration.

\begin{figure}
	\includegraphics[width=\linewidth]{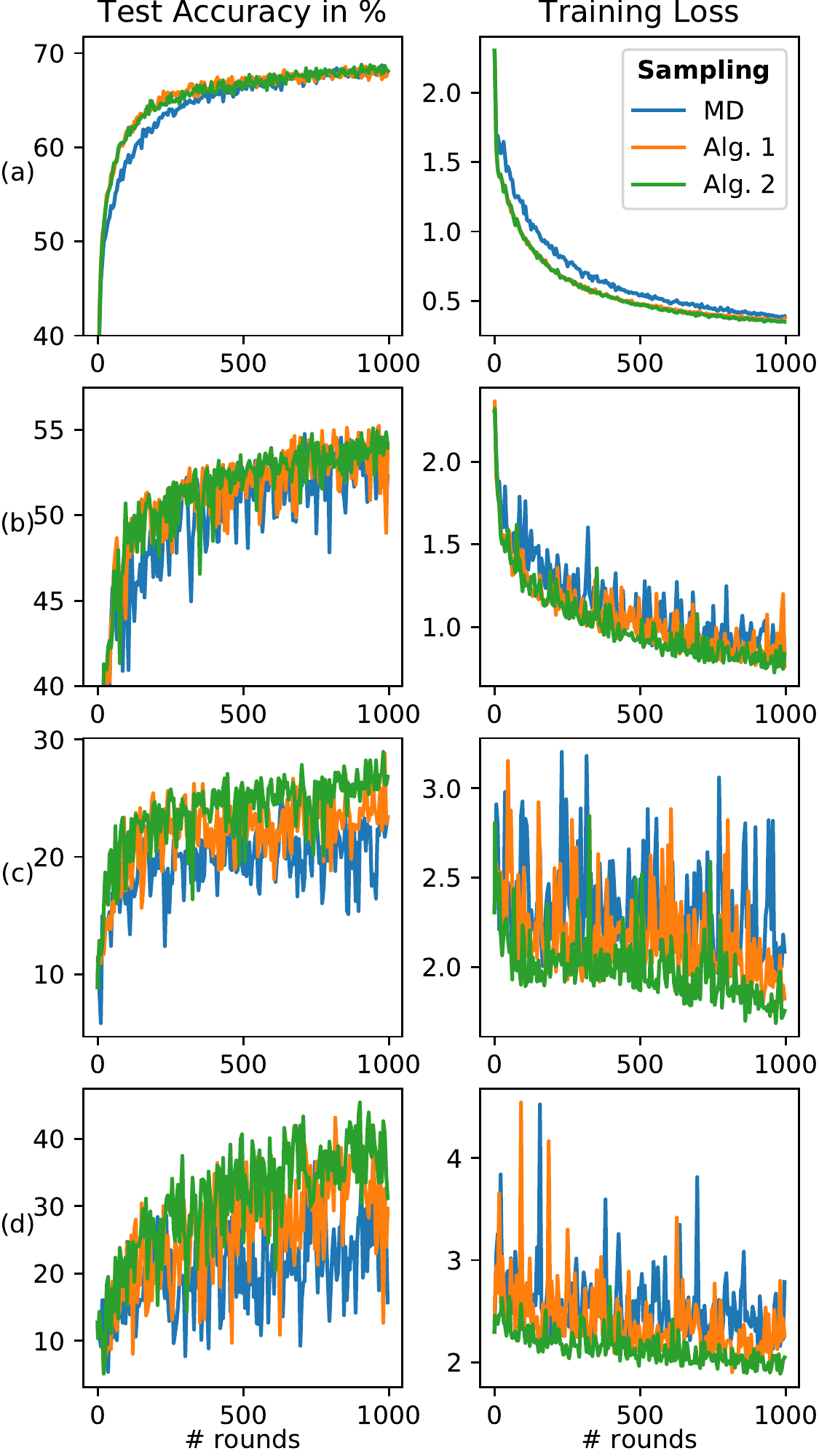}
	\caption{We investigate the improvement provided by clustered sampling on federated unbalanced datasets partitioned from CIFAR10 using a Dirichlet distribution with parameter $\alpha\in \{0.001, 0.01, 0.1, 10\}$ for respective row (a), (b), (c), (d). We use $N=100$, $m=10$, and respective learning rate for each dataset $lr=\{0.05, 0.05, 0.05, 0.1\}$.}
	\label{app:fig:unbalanced_full}
\end{figure}

\subsection{Influence of $m$ the number of sampled clients, and $N$ the number of SGD run}

We also investigates the influence the number of sampled clients $m$ and the number of SGD run $N$ have on the FL convergence speed and smoothness in Figure \ref{fig:effect_N_and_m}. We notice that the more important the amount of local work $N$ is, and the faster clustered sampling convergence speed is. With more local work, clients better fit their data. In non-iid dataset this translates in more forgetting on the classes and samples which are not part of the sampled clients. Regarding the amount of sampled clients $m$, we notice that with a smaller amount of sampled clients the improvement of clustered sampling over MD sampling is more important. We associate this result to the better data representativity of clustered sampling. For the same reason, when we increase the number of sampled clients, we see faster convergence for both MD and clustered sampling. The performance of clustered sampling is closer but still better than the one of MD sampling.

For sake of clarity, we note that the training loss displayed in Figures \ref{fig:effect_N_and_m} is computed as the rolling mean over 50 iterations. In Figure \ref{app:fig:unbalanced_full}, we provide the raw training global loss with the testing accuracy at every server iteration.

\begin{figure}
	\includegraphics[width=\linewidth]{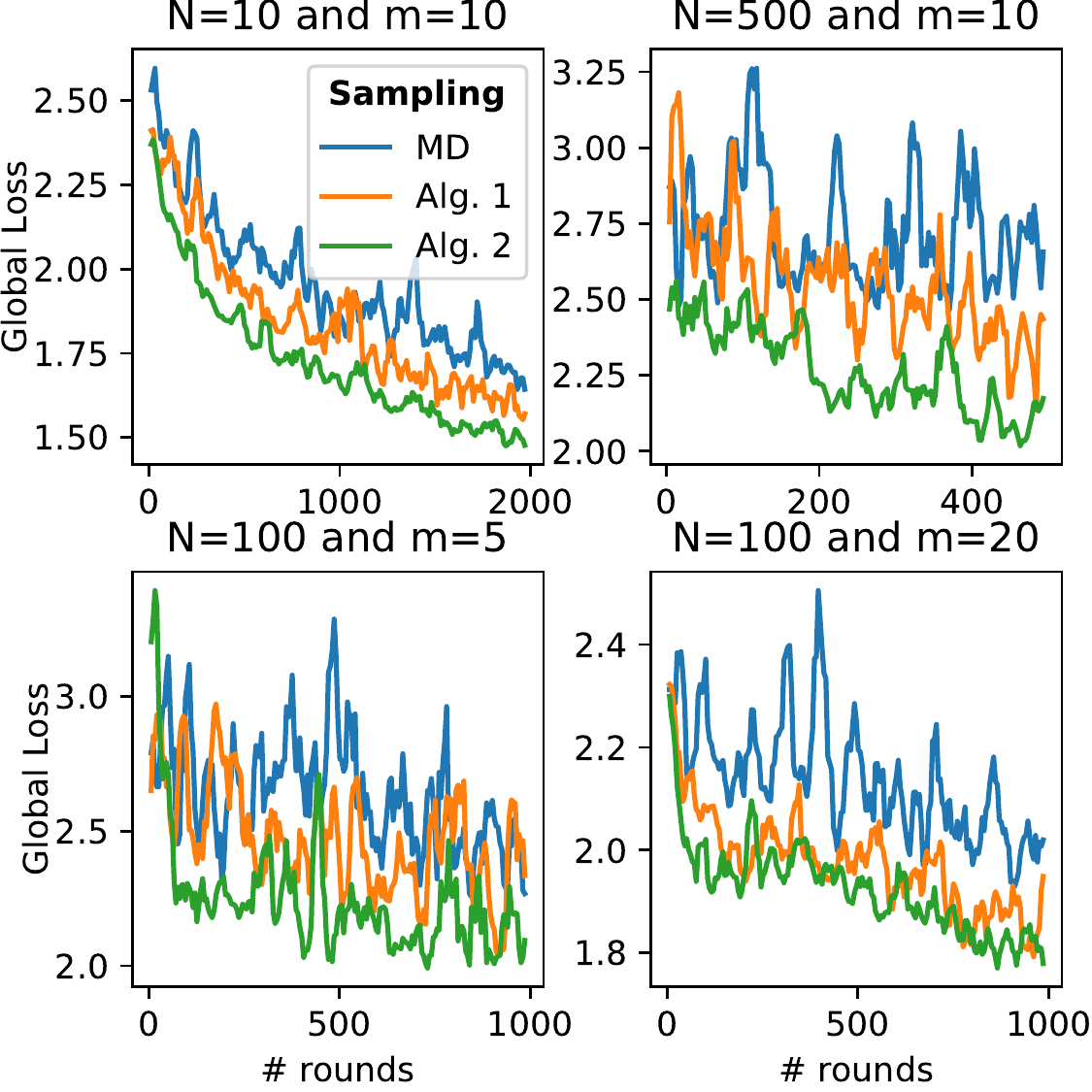}
	\caption{ We consider the federated dataset partitioned from CIFAR10 using a Dirichlet distribution with parameter $\alpha = 0.01$. We investigate the influence of $N$, the number of SGD run by each client, and $m$, the  number of sampled clients, on the training loss convergence. For each plot, experiments in first row use respectively $lr=\{0.1, 0.05\}$ and for second row $lr=\{0.05, 0.05\}$.}
	\label{fig:effect_N_and_m}
\end{figure}

\begin{figure}
	\includegraphics[width=\linewidth]{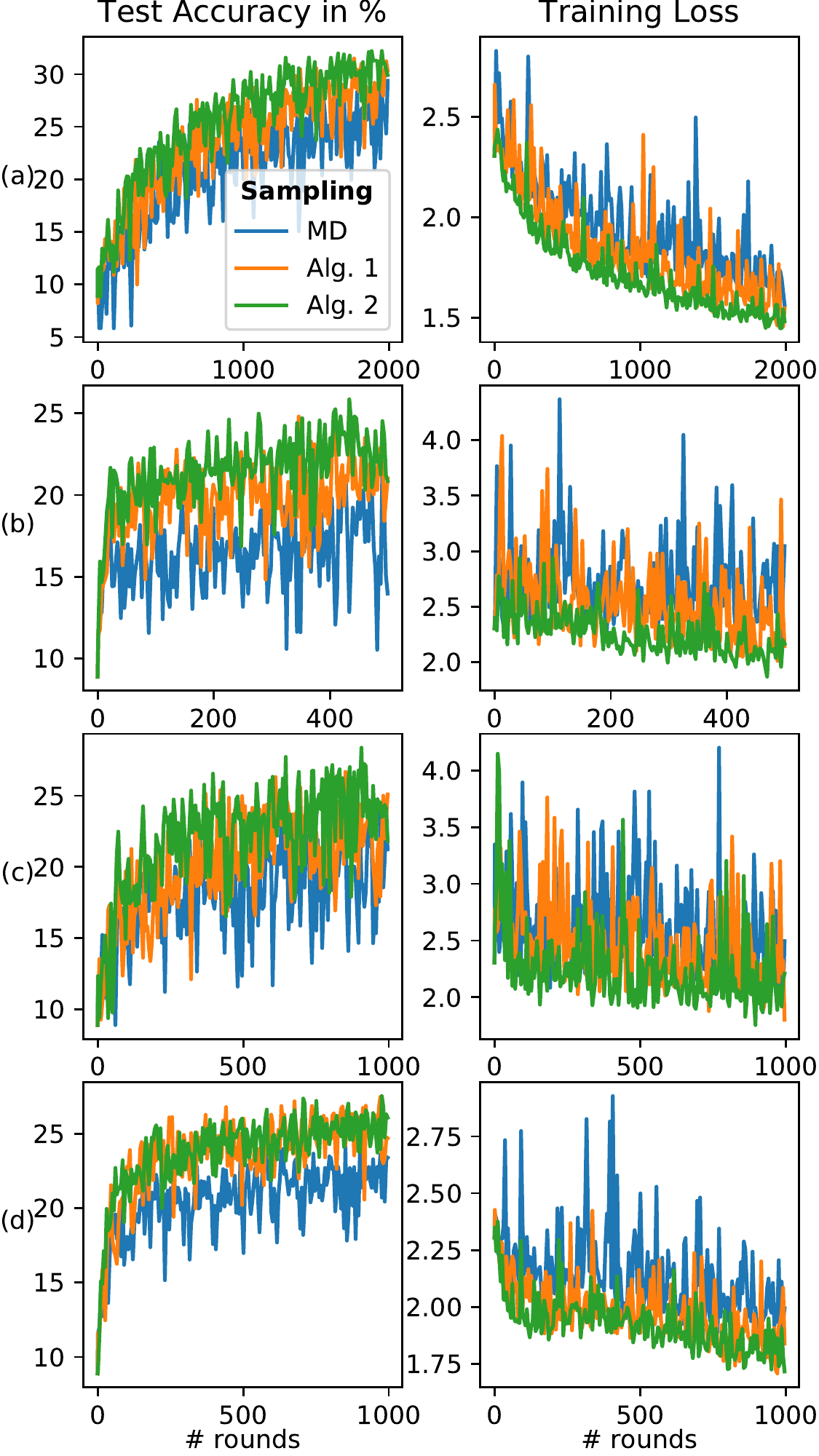}
	\caption{We consider the federated dataset partitioned from CIFAR10 using a Dirichlet distribution with parameter $\alpha = 0.01$. We investigate the influence of $N$ the number of SGD run by each client in the first two rows with $N=10$ and $N=500$ for $m=10$ and the influence of sampled clients with $m=5$ and $m=20$ for $N=100$ in the last two rows. For each dataset, we use respective learning rate $lr=\{0.1, 0.05, 0.05, 0.05\}$.}
	\label{app:fig:effects_N_and_m_full}
\end{figure}

\subsection{Local regularization}
With FedProx \cite{FedProx}, every client's local loss function is equipped with a regularization term forcing the updated model to stay close to the current global model, i.e.
\begin{equation}
    \Lcal_i'(\theta_i^{t+1})
    = \Lcal_i(\theta_i^{t+1}) + \frac{\mu}{2}  \norm{\theta_i^{t+1} - \theta^t}^2
\end{equation}
where $\theta_{t+1}$ is the updated local model of client $i$ and $\theta^t$ is the current global model. $\mu$ is the hyperparameter monitoring the regularization and is common for all the clients. This framework enables smoother federated learning processes.

We try a range of regularization term $\mu \in \{0.001, 0.01, 0.1, 1.\}$ and keep $\mu = 0.1$ maximizing the performances of FedAvg with regularization and MD sampling. We notice in Figure \ref{app:fig:CIFAR_regularization} that Algorithm \ref{alg:clustered_basic} and \ref{alg:improved_clus_sampling} still outperform MD sampling.

\begin{figure}
	\includegraphics[width=\linewidth]{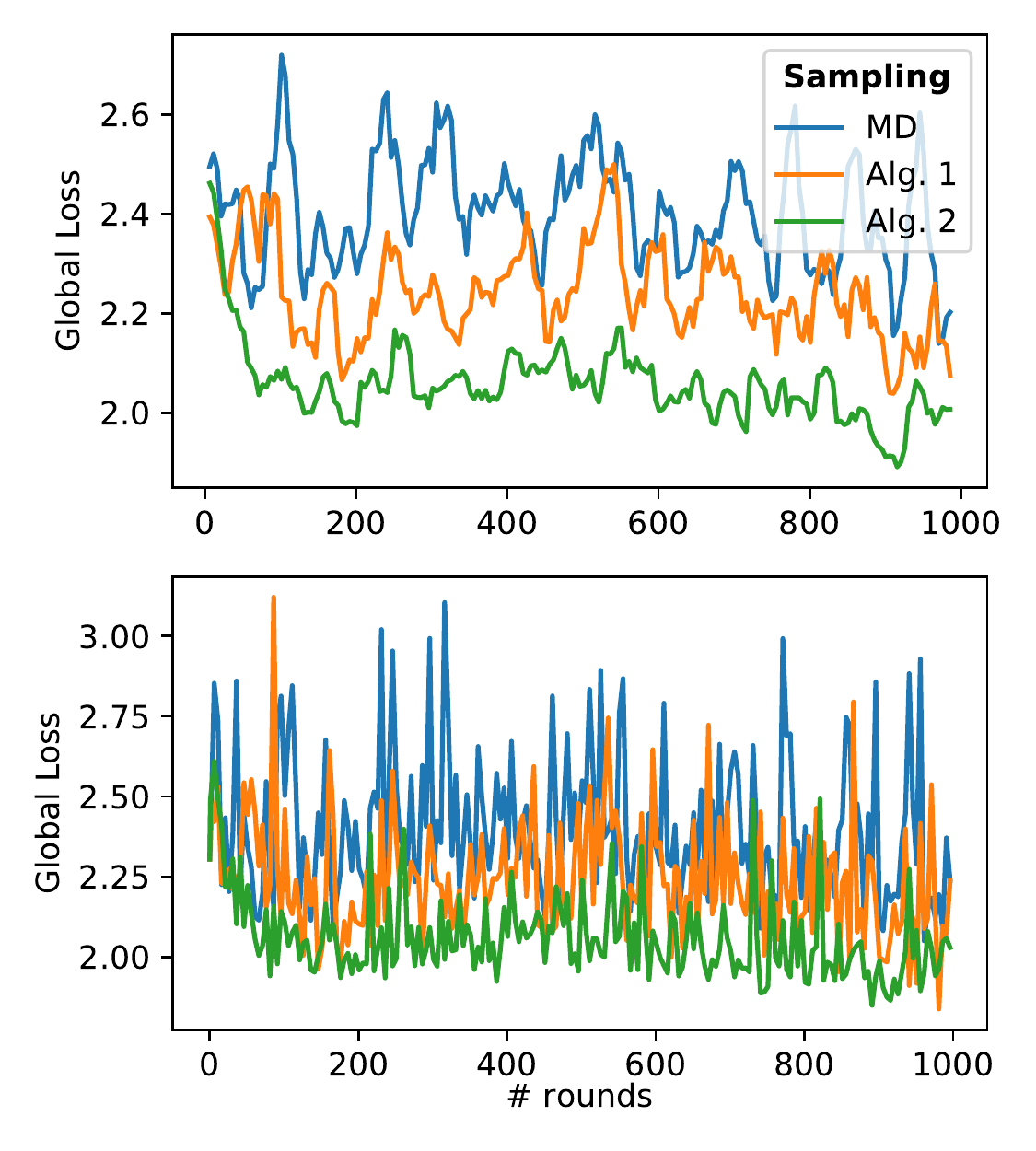}
	\caption{Training loss convergence for FL with FedProx local loss function regularization ($\mu=0.1$). We consider CIFAR10 partitioned with Dir($\alpha = 0.01$), learning rate $lr=0.05$, $m=10$ sampled clients, and $N=100$ SGD.}
	\label{app:fig:CIFAR_regularization}
\end{figure}

\end{document}